\pdfoutput=1
\documentclass[11pt]{article}

\usepackage{latex/acl}

\usepackage{times}
\usepackage{latexsym}

\usepackage[T1]{fontenc}
\usepackage[utf8]{inputenc}

\usepackage{microtype}

\usepackage{inconsolata}

\usepackage{times}  % DO NOT CHANGE THIS
\usepackage{helvet}  % DO NOT CHANGE THIS
\usepackage{courier}  % DO NOT CHANGE THIS
\usepackage{natbib}  % DO NOT CHANGE THIS AND DO NOT - ANY OPTIONS TO IT
\usepackage{caption} % DO NOT CHANGE THIS AND DO NOT - ANY OPTIONS TO IT
\usepackage{bm}
\usepackage{helvet}
\usepackage{courier}
\usepackage{graphicx}
\usepackage{bbding}    
 \usepackage{amssymb}

\usepackage{soul}
\usepackage{soul}
\usepackage[utf8]{inputenc}
\usepackage{graphicx}
\usepackage{amsmath}
\usepackage{amsthm}
\usepackage{booktabs}
\usepackage{algorithm}
\usepackage{algorithmic}
\usepackage{hyperref}
\usepackage{soul}
\usepackage[utf8]{inputenc}
\usepackage{graphicx}
\usepackage{amsmath}
\usepackage{amsthm}
\usepackage{booktabs}
\usepackage{algorithm}
\usepackage{algorithmic}
\usepackage[switch]{lineno}
\usepackage{color}
\usepackage{array}
\usepackage{booktabs}
\usepackage{multirow}
\usepackage{bm}
\usepackage{enumitem}
\usepackage{latexsym}
\usepackage{array}
\usepackage{bm}
\usepackage{helvet}
\usepackage{courier}
\usepackage{graphicx}
\usepackage{soul}
\usepackage{soul}
\usepackage[utf8]{inputenc}
\usepackage{graphicx}
\usepackage{amsmath}
\usepackage{amsthm}
\usepackage{booktabs}
\usepackage{algorithm}
\usepackage{algorithmic}
\usepackage{soul}
\usepackage[utf8]{inputenc}
\usepackage{graphicx}
\usepackage{amsmath}
\usepackage{amsthm}
\usepackage{booktabs}
\usepackage{algorithm}
\usepackage{algorithmic}
\usepackage[switch]{lineno}
\usepackage{color}
\usepackage{array}
\usepackage{booktabs}
\usepackage{multirow}
\usepackage{bm}
\usepackage{enumitem}
\usepackage{mathrsfs}
\usepackage{booktabs}
\usepackage{pifont}
\usepackage{utfsym}
\usepackage{fontawesome}

\usepackage{hyperref}       % hyperlinks
\usepackage{url}            % simple URL typesetting
\usepackage{booktabs}       % professional-quality tables
\usepackage{amsfonts}       % blackboard math symbols
\usepackage{nicefrac}       % compact symbols for 1/2, etc.
\usepackage{microtype}      % microtypography
\usepackage{graphicx} 
\usepackage{subcaption} 
\usepackage{amsmath}
\usepackage{mathtools}
\usepackage{algorithm}
\usepackage{newfloat}
\usepackage{listings}
\newtheoremstyle{mystyle}  % 主定理样式
  {6pt}   % 上方间距
  {6pt}   % 下方间距 
  {\itshape}  % 正文字体
  {}     % 缩进
  {\bfseries} % 标题字体
  {.}    % 标题后标点
  { }    % 标题后间距
  {}     % 自定义头部

\newtheoremstyle{defstyle}  % 定义专用样式
  {6pt}
  {6pt}
  {\normalfont}  % 正体字体
  {}
  {\bfseries}
  {.}
  { }
  {}

\theoremstyle{mystyle}
\newtheorem{theorem}{Theorem}[section]  % 按章节编号
\newtheorem{lemma}[theorem]{Lemma}      % 共享定理计数器

\theoremstyle{defstyle}
\usepackage{etoolbox}
\patchcmd{\proof}{\itshape}{\itshape\leavevmode}{}{}
\patchcmd{\proof}{\pushQED{\qed}}{$\hspace*{2em}\pushQED{\qed}$}{}{}

\title{Compression Hacking: A Supplementary Perspective on Informatics Metric of Language Models from Geometric Distortion}

\author{
\textbf{Jianxiang Zang}$^{1}$, 
\textbf{Meiling Ning}$^{2}$, 
\textbf{Yongda Wei}$^{3}$, 
\textbf{Shihan Dou}$^{1}$, 
\textbf{Jiazheng Zhang}$^{1}$, 
\\
\textbf{Nijia Mo}$^{3}$\textbf{,}
\textbf{Binhong Li}$^{4}$\textbf{,} 
\textbf{Tao Gui}$^{1}$\thanks{\hspace{0.5em}Corresponding author}\textbf{,} 
\textbf{Qi Zhang}$^{1}$\textbf{,} 
\textbf{Xuanjing Huang}$^{1}$
\\
\normalsize{$^{1}$ Computation and Artificial Intelligence Innovative College, Fudan University,} \\ 
\normalsize{$^{2}$ Beijing University of Posts and Telecommunications} \\
\normalsize{$^{3}$ Shanghai University of International Business and Economics} \\ 
\normalsize{$^{4}$ Hong Kong University of Science and Technology (Guangzhou)} \\
\normalsize{\texttt{jxzang25@m.fudan.edu.cn, tgui@fudan.edu.cn}}
}
\begin{document}

\maketitle

\begin{abstract}
Recently, the concept of ``compression as intelligence'' has provided a novel informatics metric perspective for language models (LMs), emphasizing that highly structured representations signify the intelligence level of LMs. However, from a geometric standpoint, the representation space of highly compressed LMs tends to degenerate into a highly anisotropic state, which hinders the LM's ability to comprehend instructions and directly impacts its performance. We found this compression-anisotropy synchronicity is essentially the ``\textbf{\emph{Compression Hacking}}'' in LM representations, where noise-dominated directions tend to create the illusion of high compression rates by sacrificing spatial uniformity.
Based on this, we propose three refined compression metrics by incorporating geometric distortion analysis and integrate them into a self-evaluation pipeline. The refined metrics exhibit strong alignment with the LM's comprehensive capabilities, achieving Spearman correlation coefficients above 0.9, significantly outperforming both the original compression and other internal structure-based metrics. This confirms that compression hacking substantially enhances the informatics interpretation of LMs by incorporating geometric distortion of representations~\footnote{Codes: https://github.com/hggzjx/compression\_hacking}.
\end{abstract}

\section{Introduction}

 Recently, significant efforts have been devoted to exploring the mechanisms by which language models (LMs) process information internally, driving the development of LM self-evaluation~\cite{wei2024diff,wang2024embedding,wang2024latent} independent of specific tasks and model outputs.
 The concept of ``compression as intelligence''~\cite{sutskever2023compressors,deletanglanguage,chen2025information} has provided a novel informatics interpretation for LMs, emphasizing that LMs eliminate redundant information through training while their representation spaces typically evolve from disordered to structured states. This property leads to a compression-based evaluation metric for LMs that utilizes differential entropy of representations, aiming to reflect model capabilities with their internal structural organization~\cite{pichler2022differential,zhouyin2023understanding,li2025semantic}. Existing studies have demonstrated strong alignment between this metric and LM scale~\cite{wei2024diff,li2025semantic}, which we have also empirically validated. However, as evidenced by the intuitive case where 175B GPT-3~\cite{brown2020language} exhibits inferior overall capabilities compared to 32B Qwen2.5-Instruct~\cite{hui2024qwen2}, compression from a purely informatics standpoint, cannot fully align with LM capabilities, especially when comparing models from different families. Therefore, our research motivation is: \emph{Beyond information compression, what other properties should a metric quantify to effectively interpret the LMs' intelligence level, and how should we model the relationships between these properties?}

Relevant studies have shown that differences in model architecture and training paradigms inevitably lead to variations in the geometric structure of representations~\cite{mimno2017strange,gao2019representation,skean2025layer}. From a geometric standpoint, we were surprised to observe that LMs with high information compression tend to exhibit representation spaces that degenerate into highly anisotropic, distorted states. Highly anisotropic representations indicate varying sensitivity to semantic changes across different dimensions, which can hinder language models' ability to comprehend instructions and consequently degrade their performance~\cite{demeter2020stolen,yu2022rare,rudman2024stable}.

In this study, we quantitatively analyze this compression-anisotropy synchronicity and validate its statistical significance. Through mechanistic analysis, we find that this phenomenon reflects the ``\textbf{\emph{Compression Hacking}}'' in LM representations, where \emph{noise-dominated directions tend to create the illusion of high compression rates by sacrificing spatial uniformity}. According to this characteristic, we propose the integration of geometric perspective to refine the information compression metric. Specifically, we introduce the following strategies: (1) a \emph{spectral entropy quantification} compression metric to model the properties of eigenvalue distributions; (2) a \emph{semantic coefficient of variation} to measure anisotropy relative to compression; and (3) a \emph{manifold correction protocol} that uses Principal Component Smoothing (PCS) as an ``anisotropy razor'' to decouple the influence of anisotropy on compression. These refined metrics are integrated into a self-evaluation pipeline that relies entirely on the LM's internal structure.  

Using this framework, we evaluate 18 open-source LMs and conduct meta-evaluations on factuality, reasoning, math, and knowledge tasks to obtain ground-truth capability scores. Extensive experiments demonstrate that the refined metrics exhibit strong alignment with the LM's comprehensive capabilities, achieving Spearman correlation coefficients above 0.9, which significantly outperforms both the original compression and other internal structure-based metrics. This validates that compression hacking substantially enhances the informatics interpretation of LMs by incorporating geometric distortion analysis of representations.
The main contributions are summarized as follows:

\begin{itemize}[left=0pt]  
    \item We introduce a significant characteristic in LM representations termed ``compression hacking'', which complements the concept of ``compression as intelligence'' from the perspective of geometric distortion.
    \item According to compression hacking, we propose three refinements of compression metrics incorporating geometric insights: spectral entropy quantification, semantic coefficient of variation, and manifold correction protocol.  
    % \item We integrate these refined metrics into a self-evaluation pipeline. Extensive experiments demonstrate strong alignment between the refined metrics and the models' comprehensive capabilities. 
    % Our evaluation toolkit available: \textcolor{blue}{https://github.com/hggzjx/compression\_hacking}.
    \item The refined metrics exhibit significantly stronger alignment with LM's comprehensive capabilities compared to the original compression metric, thereby establishing a task-agnostic self-evaluation perspective for LMs.
\end{itemize}

\section{Compression Hacking}

In this section, we analyze the compression-anisotropy synchronicity in LM representations, where highly compressed LMs tend to exhibit in-context representations with strong anisotropy. Our investigation proceeds in two stages: First, we quantify both compression and anisotropy metrics by examining the internal structure of LM representations (covariance matrices). We then fit regression curves to model the relationship between anisotropy and compression, verifying its statistical significance. Second, through mechanistic analysis, we identify the underlying cause of this phenomenon, what we term ``compression hacking''.

The covariance matrix of LM representations reflects their internal structure. For the hidden states \( \mathbf{Z} = \{\mathbf{z}(\bm{w})|\bm{w}\in \mathcal{V}\} \), where \( \bm{w} \) represents a word and \( \mathcal{V} \) represents the sample vocabulary space, the construction of the covariance matrix is as formulated in Eq.~\ref{eq.sigma}. Here, \( \mathbf{z}(\bm{w}) \in \mathbb{R}^{D}\) represents the token embeddings, which has been normalized. \(\mathbf{Z} \) is a zero-mean matrix. 

\begin{equation}
\Sigma_{\mathbf{Z}} = \frac{1}{|\mathcal V|} \mathbf{Z}^{\top}\mathbf{Z}+\alpha \mathbf{I}_{D}\label{eq.sigma}  
\end{equation}    

Here, \( \Sigma_{\mathbf{Z}} \in \mathbb{R}^{ D\times D} \) denotes the covariance matrix, and a regularization term \( \alpha \mathbf{I}_{D} \) is added to ensure it is full rank. The matrix \(\Sigma_{\mathbf{Z}}\) is positive definite and can be decomposed using eigenvalue decomposition as \(\Sigma_\mathbf{Z} = \mathbf{Q} \Lambda \mathbf{Q}^\top\). The eigenvalues from \( \Lambda \) are \( \{\lambda_d\}_{d=1}^D \), arranged in descending order by default, and \( \{\mathbf{q}_d\}_{d=1}^D \) are the corresponding eigenvectors.

\subsection{Preliminary: Differential Entropy based Compression Metric}

The compression perspective provides an information-theoretic foundation for LM evaluation, revealing the intrinsic connections between model scale, generalization capability, and data volume, thus offering theoretical guidance for optimizing model design~\cite{pichler2022differential,sutskever2023compressors,deletanglanguage,wei2024diff,chen2025information}. Related studies have shown that the differential entropy \(\mathcal{H}_{\text{DE}}(\mathbf{Z}) = -\mathbb{E}_{\bm{w}\sim\mathcal{V}} \mathbf{z}(\bm{w})\log\mathbf{z}(\bm{w})\) of LM representations \(\mathbf{z}(\bm{w})\) can reflect their compression capacity~\cite{cheninside,zhouyin2023understanding,li2025semantic}. Lower differential entropy suggests that the representations formed by nonlinear transformation, which removes redundant information, are closer to optimal coding. These representations exhibit more concentrated distributions and lower uncertainty, reflecting more efficient information compression~\cite{deletang2023language}. Semantic Volume leverages this property to model representation uncertainty~\cite{li2025semantic}.  

We thus define compression metric as the negative differential entropy of representations (i.e., \(\mathcal{C}_{\text{DE}}(\mathbf{Z})\overset{\rm def}{=}-\mathcal{H}_{\text{DE}}(\mathbf{Z})\)). Since the differential entropy is equivalent to the logdet estimator~\cite{cheninside} of their covariance matrix, the compression metric follows the definition in Eq.~\ref{eq.compression}.  

\begin{equation}
\begin{aligned}
\mathcal{C}_{\text{DE}}(\mathbf{Z})&\overset{\rm def}{=} -\frac{1}{2}\text{logdet}\left(\Sigma_{\mathbf{Z}}\right)=-\frac{1}{2}\sum_{d=1}^D\log\lambda_d\label{eq.compression}
\end{aligned}
\end{equation}

% \begin{wrapfigure}{r}{0.7\textwidth}  
%     \includegraphics[width=\linewidth]{pictures/scatter2.pdf}  
%     \caption{The compression and factuality scores of different models}
%     \label{fig.scatter}
% \end{wrapfigure}

\begin{figure}[h]
\centering
\includegraphics[width=1\linewidth]{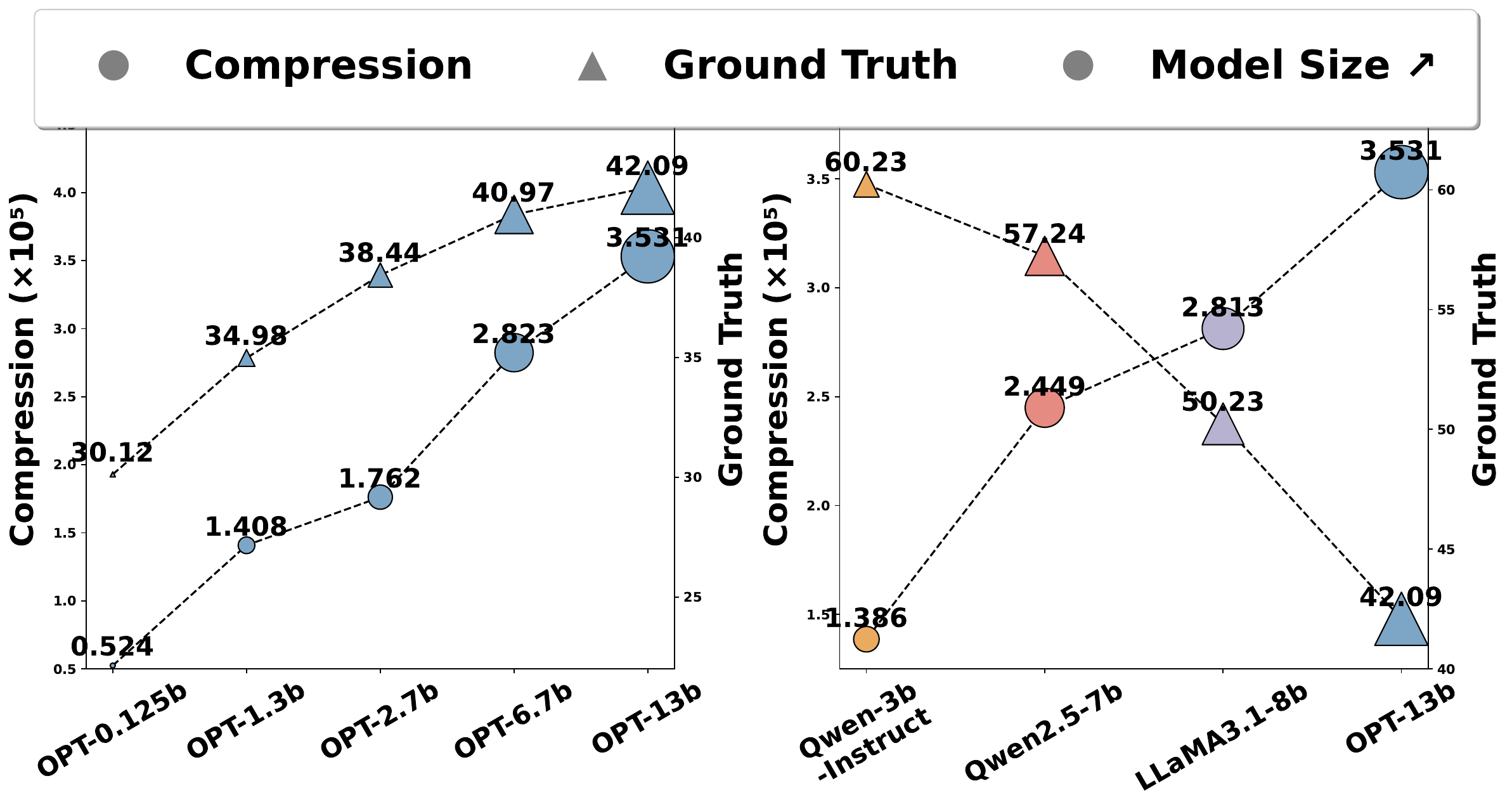}
\caption{Comparison of compression metrics across different models and their corresponding ground-truth comprehensive capabilities, categorized into intra-family and cross-family comparisons.}\label{fig.scatter}
\end{figure}

We first conducted preliminary exploration to assess whether differential entropy-based compression metrics effectively reflect LM capabilities. Our evaluation included both intra-family (OPT family) and cross-family tests (Qwen2.5-3b-Instruct, Qwen2.5-7b, LLaMA3.1-8b, and OPT-13b), with ground-truth settings following Section~\ref{sec.setup}. As shown in Figure~\ref{fig.scatter}, we found that compression metrics showed only positive correlations with model scale, consistent with related studies~\cite{wei2024diff,li2025semantic}.
However, Figure~\ref{fig.scatter}(left) indicates that differential entropy-based compression is effective only for intra-family evaluation, while Figure~\ref{fig.scatter}(right) reveals its limited applicability across diverse architectures and training paradigms. These findings prompted our integration of geometric properties into compression analysis.

\begin{figure*}[t]
\centering
\includegraphics[width=0.95\textwidth]{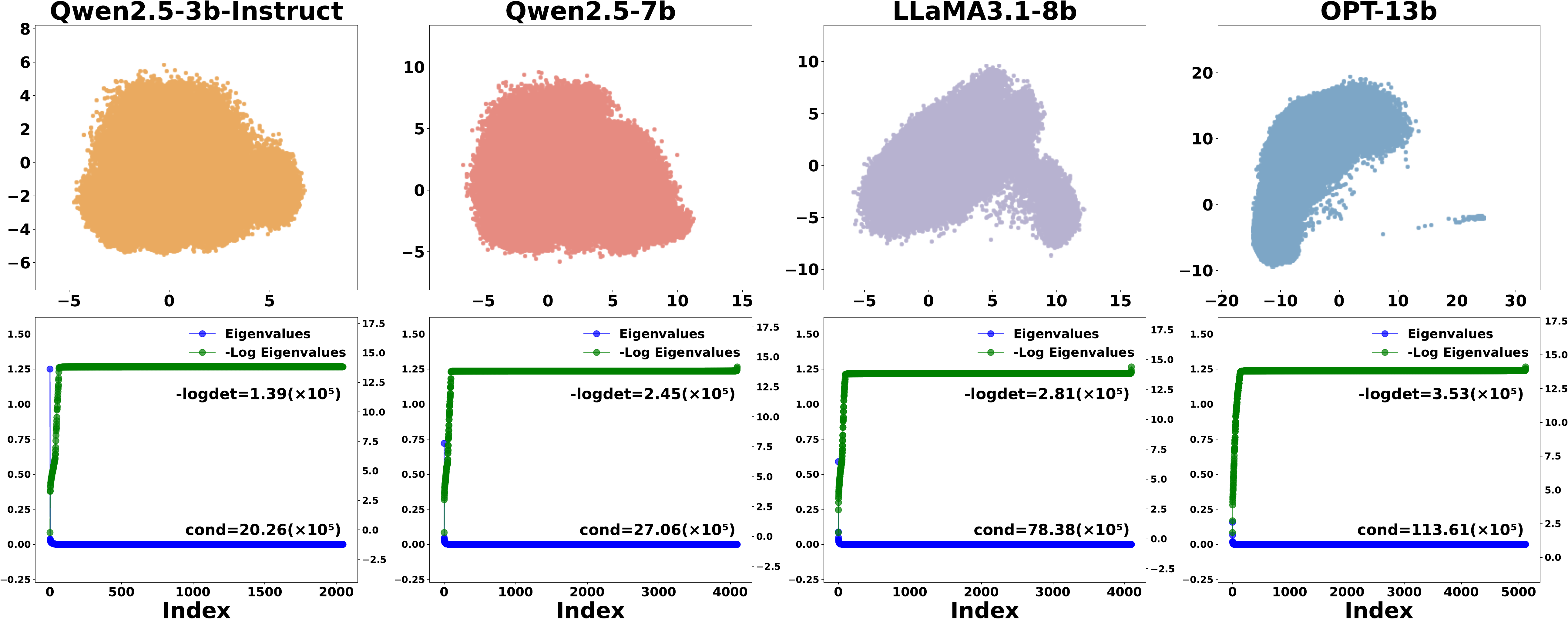}
\caption{Visualization of distribution of in-context representations and the eigenvalues across different models.}\label{fig.image2}
\end{figure*}

\subsection{Anisotropy: The Geometric Property Correlated with Compression}

The anisotropy of language models is a geometric property of representations that reflects the non-uniform distribution of semantics across different directions in the representation space~\cite{ethayarajh2019contextual,cai2021isotropy,demeter2020stolen}. Highly anisotropic representations hinder LMs' ability to comprehend instructions~\cite{yu2022rare,rudman2024stable}, directly impairing their overall capabilities. We performed principal component analysis to visualize the in-context representation spaces of the aforementioned four models. As shown in Figure~\ref{fig.image2}, we made the intriguing observation that models with higher compression levels consistently exhibited greater unevenness in their dimensional distributions, namely, higher anisotropy. This suggests a potential synergistic relationship between compression and anisotropy. If we can quantify this relationship and confirm its statistical significance, it could provide valuable guidance for refining compression metrics.

Current tools for qualitatively and quantitatively analyzing the anisotropy of language models mainly rely on similarity computations of representations~\cite{ethayarajh2019contextual,cai2021isotropy,rudman2022isoscore}. However, what we need is an anisotropy metric that can establish a connection with entropy-based information compression. Relevant studies~\cite{arora2016latent,mu2018all} have shown that the anisotropy measure \(\mathcal{A}\) is mathematically defined as formulated in Eq.~\ref{eq.AA}. We aim to extend this measure to relate to the internal structure of representations (eigenvalues of the covariance matrix).

\begin{equation}
\mathcal{A} = \frac{\max_{\|\mathbf{c}\|=1} \mathcal{Z}(\mathbf{c})}{\min_{\|\mathbf{c}\|=1} \mathcal{Z}(\mathbf{c})}\label{eq.AA} 
\end{equation}

where \( \mathcal{Z}(\mathbf{c})=\sum_{\bm{w} \in \mathcal{V}} \exp\left(\mathbf{c}^\top \mathbf{z}(\bm{w})\right) \) is the original partition function should approximately be a constant for any unit vector \(\mathbf c\). \( \mathcal{A} \) is a number greater than 1, where larger values indicate stronger anisotropy in the representation space. Ideally, this value should be as close to 1 as possible.  
Considering that \( \arg\max_{\|\mathbf{c}\|=1} \mathcal{Z}(\mathbf{c}) \) and \( \arg\min_{\|\mathbf{c}\|=1} \mathcal{Z}(\mathbf{c}) \) do not have closed-form solutions, we attempt to approximate \( \mathcal{Z}(\mathbf{c}) \) via Taylor expansion as formulated in Eq.~\ref{eq.taylor}.    

\begin{equation}
\begin{aligned}
\mathcal{Z}(\mathbf{c}) = |\mathcal{V}| + \mathbf{1}^\top_{|\mathcal{V}|}\mathbf{Z}\mathbf{c} + \frac{1}{2}\mathbf{c}^\top\mathbf{Z}^\top\mathbf{Z}\mathbf{c} \\+ \sum_{m=3}^\infty\frac{1}{m!}\sum_{\bm{w} \in \mathcal{V}} \left(\mathbf{c}^\top \mathbf{z}(\bm{w})\right)^m\label{eq.taylor}   
\end{aligned}
\end{equation}

% \begin{equation}
% \mathbf{1}^\top_{|\mathcal{V}|} \mathbf{Z} \mathbf{c} = \sum_{\bm{w}\in \mathcal{V}} \mathbf{z}(\bm{w})^\top \mathbf{c} = (\sum_{\bm{w}\in \mathcal{V}} \mathbf{z}(\bm{w}))^\top \mathbf{c} = \mathbf{0}^\top \mathbf{c} = 0
% \end{equation}

Considering that \(\mathbf{Z}\) is zero-mean data, the mean of \(\mathbf{z}(\bm{w})\) is 0. Therefore, the linear term can also be simplified to 0, that is,
\(
\mathbf{1}^\top_{|\mathcal{V}|} \mathbf{Z} \mathbf{c} = \left(\sum_{\bm{w}\in \mathcal{V}} \mathbf{z}(\bm{w})\right)^\top \mathbf{c} = \mathbf{0}^\top \mathbf{c} = 0
\)
which will not affect the relative changes of \(\mathcal{Z}(\mathbf{c})\) in different directions.
The quadratic term involves the spectral properties of the matrix, whose eigenvalues describe the directional variability of \(\mathbf{Z}^\top\mathbf{Z}\), playing a dominant role in the changes of \(\mathcal{Z}(\mathbf{c})\) in different directions. Expanding \(\mathbf{c}\) in the eigenvector basis, we have \(\mathbf{c}=\mathbf{Q}\mathbf{u}\), where \(\|\mathbf{u}\|=\|\mathbf{c}\|=1\) and \(\{u_d\}_{d=1}^D\) are the components of \(\mathbf u\). Based on the eigenvalue decomposition, the calculation of Eq.~\ref{eq.2'} is made.

\begin{equation}
\begin{aligned}
\mathbf c^\top \mathbf{Z}^\top\mathbf{Z}\mathbf c = \left(\mathbf Q\mathbf u\right)^\top\mathbf{Z}^\top\mathbf{Z} \left(\mathbf Q\mathbf u\right)=\mathbf u^\top \Lambda \mathbf u\label{eq.2'}    
\end{aligned}
\end{equation}

Accordingly, we can further obtain the second-order estimate of \(\mathcal{A}\) as formulated in Eq.~\ref{eq.A}.

\begin{equation}
\begin{aligned}
\mathcal{A} \approx \frac{|\mathcal{V}| + \max_{\|\mathbf{c}\|=1} \frac{1}{2} \mathbf{c}^\top \mathbf{Z}^\top \mathbf{Z} \mathbf{c}}{|\mathcal{V}| + \min_{\|\mathbf{c}\|=1} \frac{1}{2} \mathbf{c}^\top \mathbf{Z}^\top \mathbf{Z} \mathbf{c}} \\= \frac{|\mathcal{V}| + \max_{\|\mathbf{u}\|=1} \frac{1}{2} \sum_d \lambda_du^2_d}{|\mathcal{V}| + \min_{\|\mathbf{u}\|=1} \frac{1}{2} \sum_d \lambda_du^2_d}\label{eq.A}
\end{aligned}
\end{equation}

When the components of the vector \(\mathbf{u}\) are entirely concentrated in the direction corresponding to the maximum (minimum) eigenvalue, \(\mathbf{u}^\top \Lambda \mathbf{u} = \max_d\lambda_d (\min_d\lambda_d)\). We observed that the anisotropy of the representation can be measured by the condition number of the matrix, as formulated in Eq.~\ref{eq.cond}. The condition number reflects the sensitivity of the covariance matrix and reveals the characteristics of ill-conditioning from an intrinsic structural perspective, making it the first anisotropy metric entirely based on internal structure.

\begin{equation}
\begin{aligned}
\mathcal{A}(\mathbf{Z})\overset{\rm def}{=}\operatorname{cond}\left(\Sigma_{\mathbf{Z}}\right)=\frac{\max_{d=1}^D\lambda_d}{\min_{d=1}^D\lambda_d}\label{eq.cond}
\end{aligned}
\end{equation}

\subsection{Systematic Analysis}

\textbf{Mechanistic Analysis} As shown in Figure~\ref{fig.image2}, by performing eigenvalue decomposition on the covariance matrix of the representations, we discovered a distinctive partitioning phenomenon in the eigenvalues of the LM covariance matrix. The leading principal components exhibit an exponential decay in eigenvalues, effectively condensing the model's core semantic information, while the numerous subsequent minor components demonstrate clustered, nearly constant low eigenvalues, forming spatially anisotropic perturbation sources. Interestingly, when measuring information compression using a negative logarithmic scale, the minor components show dramatically inflated compression metrics due to their infinitesimal original eigenvalues, creating an inverted relationship with the principal component region. This seemingly paradoxical phenomenon actually reveals the compression hacking in model representations, where \textbf{\emph{noise-dominated directions tend to create the illusion of high compression rates by sacrificing spatial uniformity}}, while in reality this ``compression'' represents either information loss or noise amplification, with truly effective information compression being exclusively accomplished by the principal components.

\noindent\textbf{Significance Analysis} Next, we analyze the significance of compression hacking, which manifests as compression-anisotropy synchronicity. Based on the aforementioned metrics, we calculated the estimates of both compression and anisotropy for instruction representations across four LMs in our preliminary experiments, both of which can be exclusively represented by the eigenvalues of the representation covariance matrix. Given their characteristic patterns, we modeled a linear regression of compression against the logarithmic values of anisotropy, as shown in Figure~\ref{fig.cac}.
The regression analysis reveals two key findings through R² and p-values: (1) \emph{compression as the dependent variable can be well and significantly explained by anisotropy}, and (2) Mann-Whitney U tests~\cite{mcknight2010mann} confirm \emph{statistically significant differences in regression curves across different models}. 
%These results strongly demonstrate that compression hacking constitutes a significant intrinsic property of LMs. 

\begin{figure}[t]
\centering
\includegraphics[width=1\linewidth]{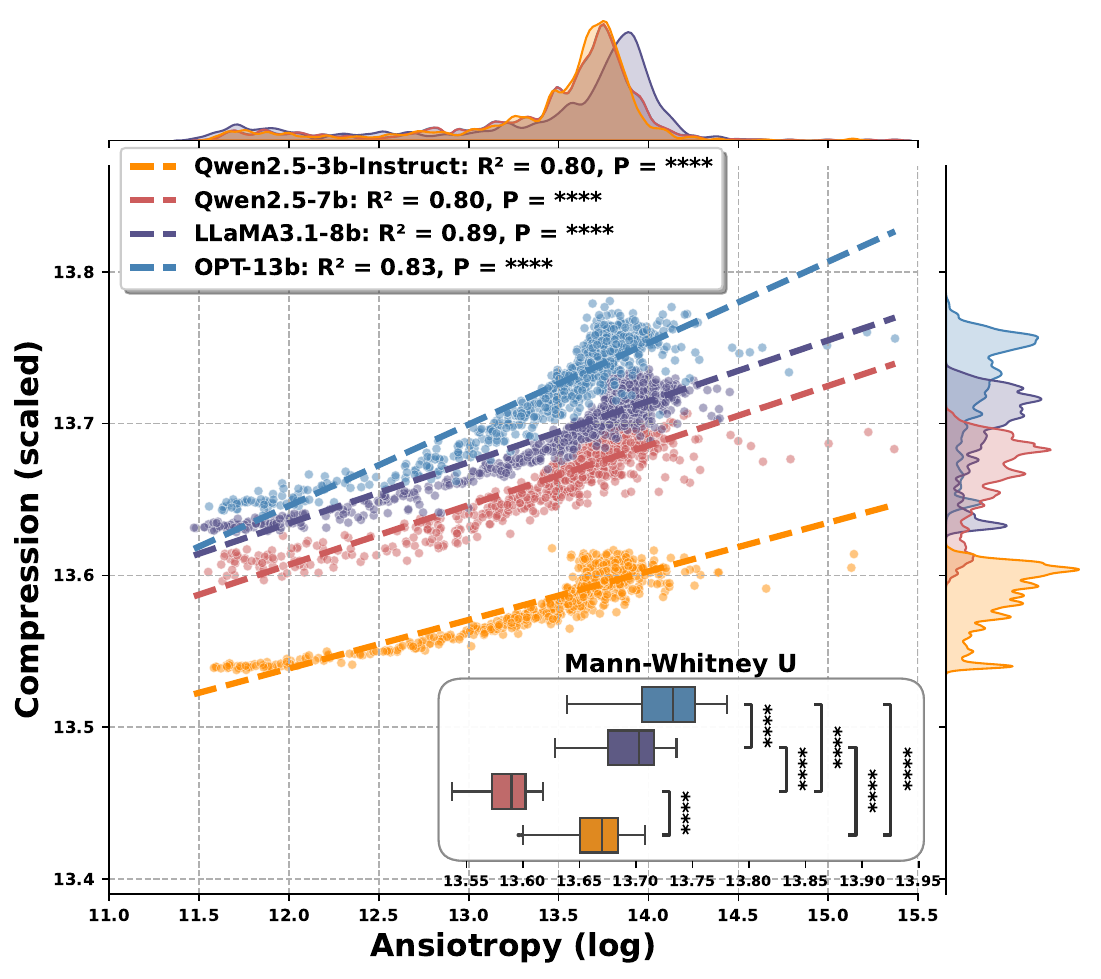}
\caption{Regression fitting curves of compression versus anisotropy for different models, along with Mann-Whitney U tests between them. Here, **** denotes statistical significance at the 0.01\% level.}\label{fig.cac}
%不同模型关于各向异性的压缩回归拟合曲线，以及它们之间的Mann-Whitney U检验
\end{figure}

\section{Methodology}

\subsection{Refined Metrics}

We have demonstrated that the compression-anisotropy synchronicity caused by compression hacking in LMs is a statistically significant characteristic. This implies that we can develop more comprehensive metrics by jointly considering the compression and anisotropy of representations, as well as modeling their correlation. In this section, we formalize our approach through three strategies:

\begin{figure*}[t]
\centering
\includegraphics[width=1\textwidth]{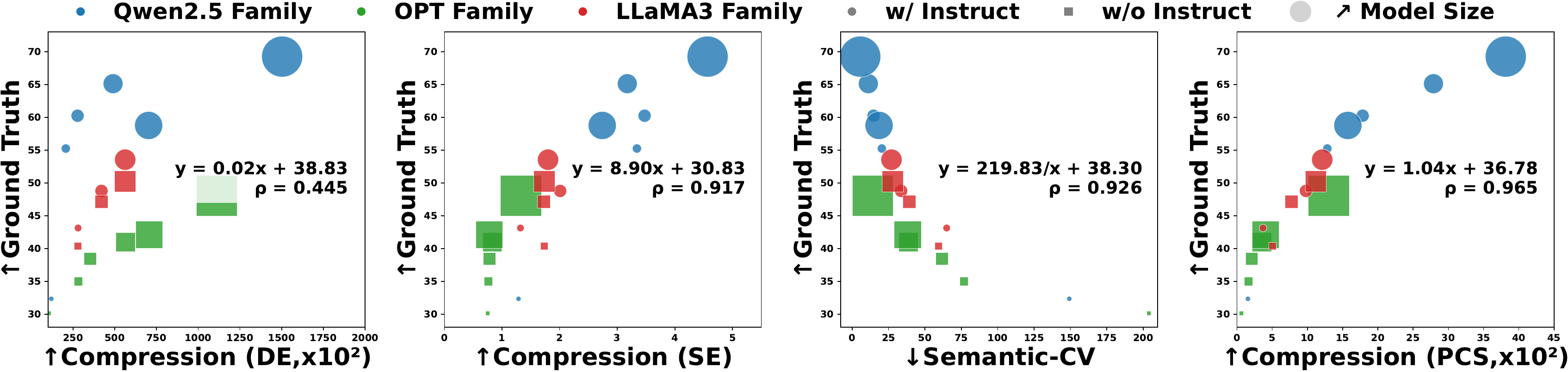}
\caption{Scatter plots of ground truth values across different models for the four metrics, along with fitted regression equations and Spearman correlation coefficients.}\label{fig.main_plot}
\end{figure*}

\begin{table*}[t]
\setlength{\tabcolsep}{4 pt}
\footnotesize
\renewcommand{\arraystretch}{1.15} 
\centering
\begin{tabular}{lcccccccc}
\hline\hline
\multirow{2}{*}{Metric} & \multicolumn{2}{c}{\textbf{Global}}                & \multicolumn{2}{c}{\textbf{Qwen2.5-Instruct}}   & \multicolumn{2}{c}{\textbf{OPT}}                & \multicolumn{2}{c}{\textbf{LLaMA3}}             \\ \cline{2-9} 
                        & \multicolumn{1}{c}{Size} & Ground truth & \multicolumn{1}{c}{Size} & Ground truth & \multicolumn{1}{c}{Size} & Ground truth & \multicolumn{1}{c}{Size} & Ground truth \\ \hline
Compression (DE)        & \multicolumn{1}{c}{0.935} & 0.445      & \multicolumn{1}{c}{1.000} & 0.829      & \multicolumn{1}{c}{1.000} & 1.000      & \multicolumn{1}{c}{0.956} & 0.886      \\ 
\textbf{Compression (SE)}        & \multicolumn{1}{c}{0.430} & 0.917      & \multicolumn{1}{c}{0.486} & 0.714      & \multicolumn{1}{c}{0.829} & 0.829      & \multicolumn{1}{c}{0.598} & 0.657      \\ 
\textbf{Semantic CV}             & \multicolumn{1}{c}{0.805} & 0.926      & \multicolumn{1}{c}{0.829} & 1.000      & \multicolumn{1}{c}{1.000} & 1.000      & \multicolumn{1}{c}{0.956} & 0.943      \\ 
\textbf{Compression (PCS)}       & \multicolumn{1}{c}{0.708} & 0.965      & \multicolumn{1}{c}{0.829} & 1.000      & \multicolumn{1}{c}{1.000} & 1.000      & \multicolumn{1}{c}{0.956} & 0.943      \\ \hline\hline
\end{tabular}
\caption{The Spearman correlation coefficients within model groups (Qwen2.5-Instruct, OPT, and LLaMA3 families) and across all models (Global), including the correlations between the four metrics, and both model size (size) and comprehensive capabilities (Ground truth). The bold-highlighted components represent our refined metrics.}\label{tab.scale}
\end{table*}

\noindent\textbf{Spectral Entropy Quantification} Figure~\ref{fig.image2} illustrates that, from the perspective of eigenvalue distribution, the mechanism of compression hacking is that the secondary components causing anisotropy (\(\lambda_d\)) are homologous to the principal components of the compression part (\(-\log\lambda_d\)). Interesting, spectral entropy~\cite{roy2007effective} precisely models this characteristic, and it is formally equivalent to a compression metric weighted by eigenvalues (Compression (SE)), as formulated in Eq.~\ref{eq.se}.

\begin{equation}
    \mathcal{C}_{\text{SE}}(\mathbf{Z})\overset{\rm def}{=}-\operatorname{tr}(\Sigma_{\mathbf{Z}}\log\Sigma_{\mathbf{Z}})=-\sum_{d=1}^D\lambda_d\log \lambda_d\label{eq.se}
\end{equation}

\noindent\textbf{Semantic Coefficient of Variation} Just as compression-anisotropy synchronicity serves as a distinct manifestation of compression hacking, where compression is characterized by the mean of eigenvalue logarithms (reflecting the overall volume of the embedding space~\cite{li2025semantic}), while anisotropy corresponds to the ratio of extreme eigenvalues (quantifying the variation of semantic embeddings across different dimensions). Thus, we formulate their ratio as the Semantic Coefficient of Variation (Semantic CV) in Eq.~\ref{eq.cv}. This metric accurately characterizes the magnitude of anisotropy relative to information compression in the representation space \(\mathbf{Z}\).

% \begin{equation}
% \mathcal{CV}_{\text{Sem.}}(\mathbf{Z}) \overset{\rm def}{=} \frac{\mathcal{C}_{\text{DE}}(\mathbf{Z})}{\mathcal{A}(\mathbf{Z})}\label{eq.cv}
% \end{equation}

\begin{equation}
\mathcal{CV}_{\text{Sem.}}(\mathbf{Z}) \overset{\rm def}{=} \frac{\mathcal{A}(\mathbf{Z})}{\mathcal{C}_{\text{DE}}(\mathbf{Z})}\label{eq.cv}
\end{equation}

\noindent\textbf{Manifold Correction Protocol} Numerous studies have proposed train-free ``anisotropy razors'' to reduce the anisotropy of representation space in a train-free manner, thereby enhancing representational capacity~\cite{mu2018all,su2021whitening}. This inspires us to decouple anisotropy from compression by selecting an appropriate anisotropy razor. Considering the exponential sharp decline in the eigenvalues of principal components corresponding to preceding dimensions due to compression hacking, we propose Principal Component Smoothing (PCS) as an anisotropy razor, inspired by the LW-shrinkage~\cite{ledoit2004well}. By setting a smoothing coefficient \(\beta \in [0,1]\) (default value is set to 0.9), we shift the representation space toward principal directions, resulting in a flatter transformed feature spectrum. This transformation is based on the covariance matrix of the representation and is achieved by defining the mapping \(\mathcal{T}_{\text{PCS}}\) as formulated in Eq.~\ref{eq.pcs}, thereby refining the compression metric Compression (PCS). In Theorem~\ref{theorem.pcs}, we prove that under sparse spectrum conditions, the PCS estimator exhibits higher statistical stability than the LW shrinkage.

\begin{equation}
    \mathcal{T}_{\text{PCS}}({\Sigma}_{\mathbf{Z}}) \overset{\rm def}{=} (1-\beta)\Sigma_{\mathbf{Z}}+\beta\max_{d=1}^D\lambda_d\mathbf{I}_D\label{eq.pcs}
\end{equation}

\subsection{Evaluation Pipeline}

In this section, we integrate the three refined metrics into a unified evaluation framework, which is a task-agnostic pipeline operating purely from a representational perspective. Our evaluation paradigm associates the sampled data batch \(\mathcal{B}\) with a decision score \(s = \mathcal{F}(\mathcal{B}, f_\text{LM})\). The decision function \(\mathcal{F}(\cdot)\) operates through two sequential processes: (1) the projection step extracts hidden representations \(\mathbf{Z}^{(\bm{p})} = \mathcal{F}_{\text{Projection}}(\bm{p}, f_{\text{LM}})\) for each data sample \(\bm{p} \in \mathcal{B}\); (2) the decision step computes the batch-level score \(s =\mathbb{E}_{\bm{p}\sim\mathcal{B}}\text{Metric}(\mathbf{Z}^{\bm{p}})\) based on the refined metrics. Notably, our dataset requirement specifies that the sample's in-context representation space should effectively estimate the model's complete in-context representation space given sufficient sampling, ensuring convergence of our proposed metrics. We discuss the impact of sampling size on metric convergence in Section~\ref{sec.eval_detail}.

\section{Experiments}

In this section, we employ meta-evaluation to investigate whether the refined metrics can achieve strong alignment with the comprehensive capabilities of LMs. This serves to validate whether incorporating the geometric distortion perspective of representations through compression hacking can enhance the informatics interpretation of LMs.

\subsection{Setup}\label{sec.setup}

\textbf{Models} Since our evaluation focuses on the internal structure of model representations, we evaluated 18 open-source language models from three different model families with varying sizes. These families are the LLaMA3 family~\cite{grattafiori2024llama} (LLaMA3.2-1B, LLaMA3.2-1B-Instruct, LLaMA3.2-3B, LLaMA3.2-3B-Instruct, LLaMA3.1-8B, LLaMA3.1-8B-Instruct), Qwen2.5-Instruct family~\cite{hui2024qwen2} (0.5B, 1.5B, 3B, 7B, 14B, 32B), OPT family~\cite{zhang2022opt} (0.125B, 1.3B, 2.7B, 6.7B, 13B, 30B). 

\noindent\textbf{Meta Evaluation} To evaluate the alignment between our metrics and LM capabilities, we employed meta-evaluation by calculating the Spearman correlation coefficient between human-annotated ground truth benchmarks and our proposed refined informatics metrics. For the meta-evaluation experiments, we selected six benchmark datasets spanning four major domains as ground truth, corresponding to four key dimensions of large language model capabilities: Factuality: TruthfulQA~\cite{lin2022truthfulqa}, FACTOR~\cite{muhlgay2024generating}, Math: MATH~\cite{hendrycks2021measuring}, Reasoning: CommonsenseQA~\cite{talmor2019commonsenseqa}, TheoremQA~\cite{chen2023theoremqa}, Knowledge: MMLU~\cite{hendrycks2020measuring}. We use the mean of all benchmark scores as the ground truth for the model's comprehensive evaluation (CE).

\begin{table*}[t]
\centering
\footnotesize
\setlength{\tabcolsep}{4 pt}
\renewcommand{\arraystretch}{1.15} 
\begin{tabular}{lccccccccc}
\hline\hline
\multicolumn{1}{c}{\multirow{2}{*}{\textbf{Metric}}} & \multicolumn{2}{c}{\textbf{Property}}                  & \multicolumn{2}{c}{\textbf{Factuality}}                      & \multicolumn{2}{c}{\textbf{Reasoning}}                         & \multicolumn{1}{l}{\textbf{Math}} & \multicolumn{1}{l}{\textbf{Knowledge}} & \multirow{2}{*}{\textbf{CE}} \\ \cline{2-9}
\multicolumn{1}{c}{}                                 & \multicolumn{1}{c}{Info.} & \multicolumn{1}{c}{Geom.} & \multicolumn{1}{c}{TruthfulQA} & \multicolumn{1}{c}{FACTOR} & \multicolumn{1}{l}{Common.QA} & \multicolumn{1}{l}{Theo.QA} & \multicolumn{1}{c}{MATH}          & \multicolumn{1}{c}{MMLU}                   &                              \\ \hline
\multicolumn{1}{l}{Semantic Volume}                  & \multicolumn{1}{c}{\checkmark}  & \multicolumn{1}{c}{}      & \multicolumn{1}{c}{0.429}      & \multicolumn{1}{c}{0.414}  & \multicolumn{1}{c}{0.441}     & \multicolumn{1}{c}{0.483}     & \multicolumn{1}{c}{0.420}         & \multicolumn{1}{c}{0.409}                  & 0.442                        \\ 
\multicolumn{1}{l}{Curvature}                        & \multicolumn{1}{c}{}      & \multicolumn{1}{c}{\checkmark}  & \multicolumn{1}{c}{0.355}      & \multicolumn{1}{c}{0.372}  & \multicolumn{1}{c}{0.342}     & \multicolumn{1}{c}{0.365}     & \multicolumn{1}{c}{0.303}         & \multicolumn{1}{c}{0.309}                  & 0.302                        \\ 
\multicolumn{1}{l}{Diff-eRank}                       & \multicolumn{1}{c}{\checkmark}  & \multicolumn{1}{c}{\checkmark}  & \multicolumn{1}{c}{0.476}      & \multicolumn{1}{c}{0.461}  & \multicolumn{1}{c}{0.494}     & \multicolumn{1}{c}{0.521}     & \multicolumn{1}{c}{0.424}         & \multicolumn{1}{c}{0.452}                  & 0.492                        \\ 
\multicolumn{1}{l}{Compression(DE)}                  & \multicolumn{1}{c}{\checkmark}  & \multicolumn{1}{c}{}      & \multicolumn{1}{c}{0.458}      & \multicolumn{1}{c}{0.488}  & \multicolumn{1}{c}{0.481}     & \multicolumn{1}{c}{0.471}     & \multicolumn{1}{c}{0.490}         & \multicolumn{1}{c}{0.471}                  & 0.482                        \\ 
\multicolumn{1}{l}{Anisotropy}                       & \multicolumn{1}{c}{}      & \multicolumn{1}{c}{\checkmark}  & \multicolumn{1}{c}{0.715}      & \multicolumn{1}{c}{0.702}  & \multicolumn{1}{c}{0.702}     & \multicolumn{1}{c}{0.792}     & \multicolumn{1}{c}{0.673}         & \multicolumn{1}{c}{0.709}                  & 0.701                        \\ \hline
\multicolumn{1}{l}{\textbf{Compression (SE)}}                 & \multicolumn{1}{c}{\checkmark}  & \multicolumn{1}{c}{\checkmark}  & \multicolumn{1}{c}{0.895}      & \multicolumn{1}{c}{0.861}  & \multicolumn{1}{c}{0.892}     & \multicolumn{1}{c}{0.921}     & \multicolumn{1}{c}{0.824}         & \multicolumn{1}{c}{0.852}                  & \textbf{0.912}                        \\ 
\multicolumn{1}{l}{\textbf{Semantic CV}}                      & \multicolumn{1}{c}{\checkmark}  & \multicolumn{1}{c}{\checkmark}  & \multicolumn{1}{c}{0.946}      & \multicolumn{1}{c}{0.905}  & \multicolumn{1}{c}{0.916}     & \multicolumn{1}{c}{0.926}     & \multicolumn{1}{c}{0.857}         & \multicolumn{1}{c}{0.917}                  & \textbf{0.926}                        \\ 
\multicolumn{10}{l}{\textbf{Compression (DE)}}                                                                                                                                                                                                                                                                                                                         \\ 
\multicolumn{1}{l}{w/ Remove Directions}             & \multicolumn{1}{c}{\checkmark}      & \multicolumn{1}{c}{\checkmark}      & \multicolumn{1}{c}{0.053}      & \multicolumn{1}{c}{0.102}  & \multicolumn{1}{c}{0.042}     & \multicolumn{1}{c}{0.142}     & \multicolumn{1}{c}{0.211}         & \multicolumn{1}{c}{0.093}                  & 0.110                        \\ 
\multicolumn{1}{l}{w/ Whitening}                     & \multicolumn{1}{c}{\checkmark}      & \multicolumn{1}{c}{\checkmark}      & \multicolumn{1}{c}{0.487}      & \multicolumn{1}{c}{0.498}  & \multicolumn{1}{c}{0.502}     & \multicolumn{1}{c}{0.482}     & \multicolumn{1}{c}{0.423}         & \multicolumn{1}{c}{0.456}                  & 0.472                        \\ 
\multicolumn{1}{l}{w/ LW Shrinkage}                  & \multicolumn{1}{c}{\checkmark}      & \multicolumn{1}{c}{\checkmark}      & \multicolumn{1}{c}{0.458}      & \multicolumn{1}{c}{0.488}  & \multicolumn{1}{c}{0.481}     & \multicolumn{1}{c}{0.471}     & \multicolumn{1}{c}{0.490}         & \multicolumn{1}{c}{0.471}                  & 0.482                        \\ 
\multicolumn{1}{l}{\textbf{w/ PCS}}                           & \multicolumn{1}{c}{\checkmark}  & \multicolumn{1}{c}{\checkmark}  & \multicolumn{1}{c}{0.962}      & \multicolumn{1}{c}{0.955}  & \multicolumn{1}{c}{0.923}     & \multicolumn{1}{c}{0.967}     & \multicolumn{1}{c}{0.846}         & \multicolumn{1}{c}{0.923}                  & \textbf{0.965}                        \\ \hline\hline
\end{tabular}
\caption{The Spearman correlation coefficient between the metrics based on the representation properties and the ground truth benchmark, where gray-highlighted components represent refined metrics we proposed.}\label{tab.main}
\end{table*}

\noindent\textbf{Baseline Metrics}
We selected purely representation-based baseline metrics that operate independently of ground-truth labels and model sampling, encompassing both informatics and geometric perspectives. The informatics metrics include Compression (DE) and Semantic Volume~\cite{li2025semantic}, while the geometric metrics consist of Curvature~\cite{hosseini2023large} quantifying manifold curvature characteristics, and anisotropy. Diff-eRank~\cite{wei2024diff} is the metric that simultaneously models both information compression and geometric structure in language model representations, yet neglecting their direct synergistic relationship.

\noindent\textbf{Baseline Anisotropy Razors} In addition to PCS as the anisotropy razor for decoupling anisotropy from compression, we selected three anisotropy razors as baselines. Remove Directions~\cite{mu2018all} is a post-processing method for eliminating noisy directions. Whitening~\cite{su2021whitening} eliminates correlations between features through global scaling, normalizing the eigenvalues to have the same mean and variance. LW Shrinkage~\cite{ledoit2004well}, on the other hand, adjusts extreme eigenvalues linearly towards the mean via Bayesian shrinkage.

\subsection{Main Results}

\begin{figure*}[t]
\centering
\includegraphics[width=0.95\textwidth]{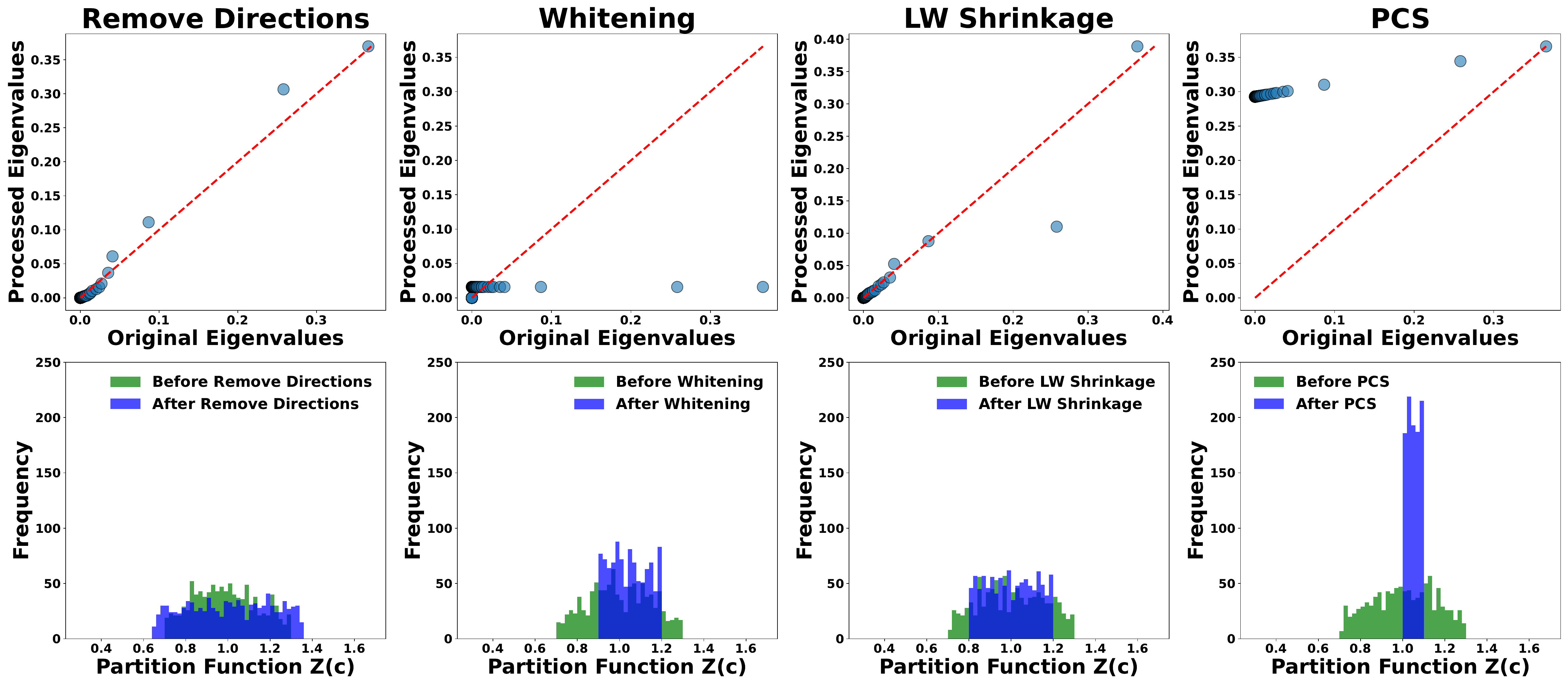}
\caption{The qqplot of the eigenvalue distribution before and after using different anisotropy razors, and the distribution of the partition function \(\mathcal{Z}(\mathbf{c})\).}\label{fig.qqplot}
\end{figure*}

Figure~\ref{fig.main_plot} and Table~\ref{tab.scale} present the regression equations and Spearman correlations among the original compression metric (compression (DE)), our three proposed refined metrics, and comprehensive capabilities as ground truth. The original compression (DE) exhibits strong correlations of 0.935 with model size across all models, reaching 1.000, 1.000, and 0.956 within model families, confirming the high consistency between original compression capability and model scale in language models. However, this metric achieves only 0.445 correlation with comprehensive capabilities in cross-architecture global analysis, maintaining higher correlations (0.829, 1.000, 0.886) only within model families, suggesting model size' applicability for LM capability assessment is confined to homogeneous architectural systems.

Among our refined metrics, compression (SE) shows reduced size correlation (0.430 globally) but achieves 0.917 cross-architecture capability correlation, demonstrating its effectiveness in capturing capability differences across diverse architectures. Both semantic CV and compression (PCS) maintain dual high correlations with size and capabilities within model families while sustaining stable cross-architecture capability correlations (0.926 and 0.965, respectively), with size correlations moderately decreasing to 0.805 and 0.708. This demonstrates that our refined metrics achieve significantly stronger alignment with LMs' comprehensive capabilities compared to the original compression metrics. Through compression hacking, we substantially enhance the informatics interpretation of LMs from the geometric distortion perspective of representations, thereby extending the ``compression as intelligence'' concept.

\subsection{Comparison with Baseline Metrics}

Table~\ref{tab.main} systematically presents the Spearman correlation coefficients between the ground truth benchmarks, and both the baseline metrics based on internal representations and our refined metrics. The property column identifies whether the metric describes informatics (Info.) or geometric (Geom.) property. Notably, metrics that model only a single property (either informational or geometric property), such as semantic volume, curvature, compression (DE), anisotropy, and their modified versions (w/ remove directions), all exhibit correlation coefficients with the comprehensive score below 0.5. Although Diff-eRank incorporates spectral entropy characteristics, its results still fail to reflect comprehensive capabilities, possibly because this metric focuses on the noise reduction process of knowledge acquisition while neglecting the synergy between information and geometric properties. Experiments show that the compression methods modified by whitening and LW shrinkage, although aiming to decouple anisotropic features, still do not significantly improve capability alignment. It is noteworthy that our refined metrics in Figure~\ref{fig.main_plot} demonstrate significant advantages over the baseline metrics.

\subsection{Effect of Anisotropy Razors}

Table~\ref{tab.main} reveals that as ``anisotropy razor'' methods, remove directions, whitening, and LW shrinkage all fail to effectively improve the reflection of comprehensive capabilities, whereas PCS exhibits a significant improvement. In this section, we investigate the structural changes in representations before and after processing with these anisotropy razors, conducting an in-depth mechanistic analysis of PCS's advantages over other methods.  
%Remove Directions simplifies data by truncating redundant features while retaining the original linear scaling of the remaining features.  Whitening eliminates feature correlations via global rescaling, normalizing eigenvalues to uniform mean and variance. LW Shrinkage linearly adjusts extreme eigenvalues toward the mean through Bayesian shrinkage.  

The qqplot in Figure~\ref{fig.qqplot} illustrates the eigenvalue distributions before and after applying these four razors. The first three methods decouple anisotropy while maintaining the linear geometric structure of the data, resulting in eigenvalues that still exhibit distinct partitioning. In contrast, PCS upscales the low-eigenvalue region, ensuring that the corrected compression relies entirely on the contributions of the principal components.  
The formal method for anisotropy detection involves examining the ``self-normalization'' property (i.e., \(\mathcal{Z} (\mathbf{c})\) tending toward a constant, independent of \(\mathbf{c}\))~\cite{mu2018all}. Figure~\ref{fig.qqplot} illustrates the distribution of \(\mathcal{Z} (\mathbf{c})\) before and after applying different anisotropy razors. We observe that remove directions leads to a more dispersed \(\mathcal{Z} (\mathbf{c})\) distribution, increasing anisotropy. This occurs because truncating certain directions causes the remaining ones to spread more extremely. In contrast, whitening, LW shrinkage, and PCS concentrate the \(\mathcal{Z} (\mathbf{c})\) distribution. Notably, PCS achieves more pronounced anisotropy elimination than the other methods by rigidly correcting the eigenvalue distribution.

\section{Conclusion}

We introduce a notable characteristic in language models termed ``compression hacking'', where the noisy directions in LM representations feign high compression rates by sacrificing spatial uniformity, thereby distorting information compression metrics. Through spectral entropy quantification, semantic coefficient of variation, and a manifold correction protocol based on principal component smoothing, we refine the compression measurement framework. Extensive experiments on 18 mainstream language models demonstrate that the refined metrics achieve strong alignment with models' actual capabilities. These results prove that incorporating the geometric distortion perspective through compression hacking significantly enhances the informatics interpretation of LMs.

\section{Limitations}

In fact, the metrics we propose still have broader application scenarios worth exploring. For instance, practical techniques such as pruning, quantization, and distillation could potentially benefit from these indicators that reveal internal redundancies. Our proposed metrics help better identify compressible components in models without causing significant information loss. We anticipate that these refined metrics may open new avenues for future research, exploring how such internal representation indicators can be applied to various potential scenarios.

% \section{Acknowledgements}

% The authors wish to thank the anonymous reviewers for their helpful comments. This work was partially funded by Guangdong S\&T Program 2024B0101050003, National Natural Science Foundation of China (No.62206057, 62076069, 61976056), Shanghai Rising-Star Program (23QA1400200), and Natural Science Foundation of Shanghai (23ZR1403500).

%\bibliographystyle{acl}
\bibliography{cv}

@inproceedings{cheninside,
  title={INSIDE: LLMs' Internal States Retain the Power of Hallucination Detection},
  author={Chen, Chao and Liu, Kai and Chen, Ze and Gu, Yi and Wu, Yue and Tao, Mingyuan and Fu, Zhihang and Ye, Jieping},
  booktitle={The Twelfth International Conference on Learning Representations},
  year={2023}
}

@inproceedings{gaorepresentation,
  title={Representation Degeneration Problem in Training Natural Language Generation Models},
  author={Gao, Jun and He, Di and Tan, Xu and Qin, Tao and Wang, Liwei and Liu, Tieyan},
  booktitle={International Conference on Learning Representations},
  year={2019}
}

@inproceedings{cai2021isotropy,
  title={Isotropy in the contextual embedding space: Clusters and manifolds},
  author={Cai, Xingyu and Huang, Jiaji and Bian, Yuchen and Church, Kenneth},
  booktitle={International conference on learning representations},
  year={2019}
}

@inproceedings{lin2022truthfulqa,
  title={TruthfulQA: Measuring How Models Mimic Human Falsehoods},
  author={Lin, Stephanie and Hilton, Jacob and Evans, Owain},
  booktitle={Proceedings of the 60th Annual Meeting of the Association for Computational Linguistics (Volume 1: Long Papers)},
  pages={3214--3252},
  year={2022}
}

@misc{sutskever2023compressors,
  author = {Ilya Sutskever},
  title = {Stronger compressors find more shared structure},
  year = {2023},
  howpublished = {The Ilya’s Talk},
  note = {Talk},
}

@inproceedings{mimno2017strange,
  title={The strange geometry of skip-gram with negative sampling},
  author={Mimno, David and Thompson, Laure},
  booktitle={Proceedings of the 2017 Conference on Empirical Methods in Natural Language Processing},
  pages={2873--2878},
  year={2017}
}

@inproceedings{demeter2020stolen,
  title={Stolen Probability: A Structural Weakness of Neural Language Models},
  author={Demeter, David and Kimmel, Gregory and Downey, Doug},
  booktitle={Proceedings of the 58th Annual Meeting of the Association for Computational Linguistics},
  pages={2191--2197},
  year={2020}
}

@inproceedings{wei2024diff,
  title={Diff-eRank: A Novel Rank-Based Metric for Evaluating Large Language Models},
  author={Wei, Lai and Tan, Zhiquan and Li, Chenghai and Wang, Jindong and Huang, Weiran},
  booktitle={The Thirty-eighth Annual Conference on Neural Information Processing Systems},
  year={2024}
}

@inproceedings{ethayarajh2019contextual,
  title={How Contextual are Contextualized Word Representations? Comparing the Geometry of BERT, ELMo, and GPT-2 Embeddings},
  author={Ethayarajh, Kawin},
  booktitle={Proceedings of the 2019 Conference on Empirical Methods in Natural Language Processing and the 9th International Joint Conference on Natural Language Processing (EMNLP-IJCNLP)},
  year={2019},
  organization={Association for Computational Linguistics}
}

@inproceedings{rudman2022isoscore,
  title={IsoScore: Measuring the Uniformity of Embedding Space Utilization},
  author={Rudman, William and Gillman, Nate and Rayne, Taylor and Eickhoff, Carsten},
  booktitle={Findings of the Association for Computational Linguistics: ACL 2022},
  pages={3325--3339},
  year={2022}
}

@inproceedings{mu2018all,
  title={All-but-the-Top: Simple and Effective Postprocessing for Word Representations},
  author={Mu, Jiaqi and Viswanath, Pramod},
  booktitle={International Conference on Learning Representations},
  year={2018}
}

@article{arora2016latent,
  title={A Latent Variable Model Approach to PMI-based Word Embeddings},
  author={Arora, Sanjeev and Li, Yuanzhi and Liang, Yingyu and Ma, Tengyu and Risteski, Andrej},
  journal={Transactions of the Association for Computational Linguistics},
  volume={4},
  pages={385--399},
  year={2016},
  publisher={MIT Press-Journals}
}

@article{zhouyin2023understanding,
  title={Understanding Neural Networks withLogarithm Determinant Entropy Estimator},
  author={Zhouyin, Zhanghao and Liu, Ding},
  year={2023}
}

@inproceedings{sellam2020bleurt,
  title={BLEURT: Learning Robust Metrics for Text Generation},
  author={Sellam, Thibault and Das, Dipanjan and Parikh, Ankur},
  booktitle={Proceedings of the 58th Annual Meeting of the Association for Computational Linguistics},
  pages={7881--7892},
  year={2020}
}

@inproceedings{lin2004rouge,
  title={Rouge: A package for automatic evaluation of summaries},
  author={Lin, Chin-Yew},
  booktitle={Text summarization branches out},
  pages={74--81},
  year={2004}
}

@misc{conover2023free,
  title={Free dolly: Introducing the world’s first truly open instruction-tuned llm},
  author={Conover, Mike and Hayes, Matt and Mathur, Ankit and Xie, Jianwei and Wan, Jun and Shah, Sam and Ghodsi, Ali and Wendell, Patrick and Zaharia, Matei and Xin, Reynold},
  year={2023}
}

@misc{wikimedia_downloads,
  author       = {Wikimedia Foundation},
  title        = {Wikimedia Downloads},
  year         = {2025},
  url          = {https://dumps.wikimedia.org/},
  note         = {Accessed: 2025-03-02}
}

@article{su2021whitening,
  title={Whitening sentence representations for better semantics and faster retrieval},
  author={Su, Jianlin and Cao, Jiarun and Liu, Weijie and Ou, Yangyiwen},
  journal={arXiv preprint arXiv:2103.15316},
  year={2021}
}

@article{zhang2022opt,
  title={Opt: Open pre-trained transformer language models},
  author={Zhang, Susan and Roller, Stephen and Goyal, Naman and Artetxe, Mikel and Chen, Moya and Chen, Shuohui and Dewan, Christopher and Diab, Mona and Li, Xian and Lin, Xi Victoria and others},
  journal={arXiv preprint arXiv:2205.01068},
  year={2022}
}

@article{hui2024qwen2,
  title={Qwen2. 5-coder technical report},
  author={Hui, Binyuan and Yang, Jian and Cui, Zeyu and Yang, Jiaxi and Liu, Dayiheng and Zhang, Lei and Liu, Tianyu and Zhang, Jiajun and Yu, Bowen and Lu, Keming and others},
  journal={arXiv preprint arXiv:2409.12186},
  year={2024}
}

@article{grattafiori2024llama,
  title={The llama 3 herd of models},
  author={Grattafiori, Aaron and Dubey, Abhimanyu and Jauhri, Abhinav and Pandey, Abhinav and Kadian, Abhishek and Al-Dahle, Ahmad and Letman, Aiesha and Mathur, Akhil and Schelten, Alan and Vaughan, Alex and others},
  journal={arXiv e-prints},
  pages={arXiv--2407},
  year={2024}
}

@inproceedings{muhlgay2024generating,
  title={Generating Benchmarks for Factuality Evaluation of Language Models},
  author={Muhlgay, Dor and Ram, Ori and Magar, Inbal and Levine, Yoav and Ratner, Nir and Belinkov, Yonatan and Abend, Omri and Leyton-Brown, Kevin and Shashua, Amnon and Shoham, Yoav},
  booktitle={Proceedings of the 18th Conference of the European Chapter of the Association for Computational Linguistics (Volume 1: Long Papers)},
  pages={49--66},
  year={2024}
}

@article{ledoit2004well,
  title={A well-conditioned estimator for large-dimensional covariance matrices},
  author={Ledoit, Olivier and Wolf, Michael},
  journal={Journal of multivariate analysis},
  volume={88},
  number={2},
  pages={365--411},
  year={2004},
  publisher={Elsevier}
}

@article{li2025semantic,
  title={Semantic Volume: Quantifying and Detecting both External and Internal Uncertainty in LLMs},
  author={Li, Xiaomin and Yu, Zhou and Zhang, Ziji and Zhuang, Yingying and Shah, Swair and Beniwal, Anurag},
  journal={arXiv preprint arXiv:2502.21239},
  year={2025}
}

@article{hosseini2023large,
  title={Large language models implicitly learn to straighten neural sentence trajectories to construct a predictive representation of natural language.},
  author={Hosseini, Eghbal and Fedorenko, Evelina},
  journal={Advances in Neural Information Processing Systems},
  volume={36},
  pages={43918--43930},
  year={2023}
}

@inproceedings{deletanglanguage,
  title={Language Modeling Is Compression},
  author={Deletang, Gregoire and Ruoss, Anian and Duquenne, Paul-Ambroise and Catt, Elliot and Genewein, Tim and Mattern, Christopher and Grau-Moya, Jordi and Wenliang, Li Kevin and Aitchison, Matthew and Orseau, Laurent and others},
  booktitle={The Twelfth International Conference on Learning Representations},
  year={2023}

}

@inproceedings{zhang2022unsupervised,
  title={Unsupervised sentence representation via contrastive learning with mixing negatives},
  author={Zhang, Yanzhao and Zhang, Richong and Mensah, Samuel and Liu, Xudong and Mao, Yongyi},
  booktitle={Proceedings of the AAAI Conference on Artificial Intelligence},
  volume={36},
  number={10},
  pages={11730--11738},
  year={2022}
}

@inproceedings{jiang2022promptbert,
  title={PromptBERT: Improving BERT Sentence Embeddings with Prompts},
  author={Jiang, Ting and Jiao, Jian and Huang, Shaohan and Zhang, Zihan and Wang, Deqing and Zhuang, Fuzhen and Wei, Furu and Huang, Haizhen and Deng, Denvy and Zhang, Qi},
  booktitle={Proceedings of the 2022 Conference on Empirical Methods in Natural Language Processing},
  pages={8826--8837},
  year={2022}
}

@inproceedings{gao2021simcse,
  title={SimCSE: Simple Contrastive Learning of Sentence Embeddings},
  author={Gao, Tianyu and Yao, Xingcheng and Chen, Danqi},
  booktitle={Proceedings of the 2021 Conference on Empirical Methods in Natural Language Processing},
  pages={6894},
  year={2021},
  organization={Association for Computational Linguistics}
}

@article{hendrycks2021measuring,
  title={Measuring mathematical problem solving with the math dataset},
  author={Hendrycks, Dan and Burns, Collin and Kadavath, Saurav and Arora, Akul and Basart, Steven and Tang, Eric and Song, Dawn and Steinhardt, Jacob},
  journal={arXiv preprint arXiv:2103.03874},
  year={2021}
}

@inproceedings{chen2023theoremqa,
  title={TheoremQA: A Theorem-driven Question Answering Dataset},
  author={Chen, Wenhu and Yin, Ming and Ku, Max and Lu, Pan and Wan, Yixin and Ma, Xueguang and Xu, Jianyu and Wang, Xinyi and Xia, Tony},
  booktitle={Proceedings of the 2023 Conference on Empirical Methods in Natural Language Processing},
  year={2023},
  organization={Association for Computational Linguistics}
}

@article{hendrycks2020measuring,
  title={Measuring massive multitask language understanding},
  author={Hendrycks, Dan and Burns, Collin and Basart, Steven and Zou, Andy and Mazeika, Mantas and Song, Dawn and Steinhardt, Jacob},
  journal={arXiv preprint arXiv:2009.03300},
  year={2020}
}

@inproceedings{talmor2019commonsenseqa,
  title={CommonsenseQA: A Question Answering Challenge Targeting Commonsense Knowledge},
  author={Talmor, Alon and Herzig, Jonathan and Lourie, Nicholas and Berant, Jonathan},
  booktitle={Proceedings of the 2019 Conference of the North American Chapter of the Association for Computational Linguistics: Human Language Technologies, Volume 1 (Long and Short Papers)},
  pages={4149--4158},
  year={2019}
}

@article{deletang2023language,
  title={Language modeling is compression},
  author={Del{\'e}tang, Gr{\'e}goire and Ruoss, Anian and Duquenne, Paul-Ambroise and Catt, Elliot and Genewein, Tim and Mattern, Christopher and Grau-Moya, Jordi and Wenliang, Li Kevin and Aitchison, Matthew and Orseau, Laurent and others},
  journal={arXiv preprint arXiv:2309.10668},
  year={2023}
}

@inproceedings{yu2022rare,
  title={Rare Tokens Degenerate All Tokens: Improving Neural Text Generation via Adaptive Gradient Gating for Rare Token Embeddings},
  author={Yu, Sangwon and Song, Jongyoon and Kim, Heeseung and Lee, Seongmin and Ryu, Woo-Jong and Yoon, Sungroh},
  booktitle={Proceedings of the 60th Annual Meeting of the Association for Computational Linguistics (Volume 1: Long Papers)},
  pages={29--45},
  year={2022}
}

@inproceedings{rudman2024stable,
  title={Stable Anisotropic Regularization},
  author={Rudman, William and Eickhoff, Carsten},
  booktitle={ICLR},
  year={2024}
}

@article{chen2025information,
  title={Information compression in the AI era: Recent advances and future challenges},
  author={Chen, Jun and Fang, Yong and Khisti, Ashish and {\"O}zg{\"u}r, Ayfer and Shlezinger, Nir},
  journal={IEEE Journal on Selected Areas in Communications},
  year={2025},
  publisher={IEEE}
}

@article{celikyilmaz2020evaluation,
  title={Evaluation of text generation: A survey},
  author={Celikyilmaz, Asli and Clark, Elizabeth and Gao, Jianfeng},
  journal={arXiv preprint arXiv:2006.14799},
  year={2020}
}

@article{tan2024can,
  title={Can i understand what i create? self-knowledge evaluation of large language models},
  author={Tan, Zhiquan and Wei, Lai and Wang, Jindong and Xie, Xing and Huang, Weiran},
  journal={arXiv preprint arXiv:2406.06140},
  year={2024}
}

@article{zheng2023judging,
  title={Judging llm-as-a-judge with mt-bench and chatbot arena},
  author={Zheng, Lianmin and Chiang, Wei-Lin and Sheng, Ying and Zhuang, Siyuan and Wu, Zhanghao and Zhuang, Yonghao and Lin, Zi and Li, Zhuohan and Li, Dacheng and Xing, Eric and others},
  journal={Advances in Neural Information Processing Systems},
  volume={36},
  pages={46595--46623},
  year={2023}
}

@article{sasaki2007truth,
  title={The truth of the F-measure},
  author={Sasaki, Yutaka and others},
  journal={Teach tutor mater},
  volume={1},
  number={5},
  pages={1--5},
  year={2007}
}

@article{mcknight2010mann,
  title={Mann-Whitney U Test},
  author={McKnight, Patrick E and Najab, Julius},
  journal={The Corsini encyclopedia of psychology},
  pages={1--1},
  year={2010},
  publisher={Wiley Online Library}
}

@inproceedings{pichler2022differential,
  title={A differential entropy estimator for training neural networks},
  author={Pichler, Georg and Colombo, Pierre Jean A and Boudiaf, Malik and Koliander, G{\"u}nther and Piantanida, Pablo},
  booktitle={International Conference on Machine Learning},
  pages={17691--17715},
  year={2022},
  organization={PMLR}
}

@article{gao2019representation,
  title={Representation degeneration problem in training natural language generation models},
  author={Gao, Jun and He, Di and Tan, Xu and Qin, Tao and Wang, Liwei and Liu, Tie-Yan},
  journal={arXiv preprint arXiv:1907.12009},
  year={2019}
}

@article{skean2025layer,
  title={Layer by Layer: Uncovering Hidden Representations in Language Models},
  author={Skean, Oscar and Arefin, Md Rifat and Zhao, Dan and Patel, Niket and Naghiyev, Jalal and LeCun, Yann and Shwartz-Ziv, Ravid},
  journal={arXiv preprint arXiv:2502.02013},
  year={2025}
}

@article{wang2024latent,
  title={Latent Space Chain-of-Embedding Enables Output-free LLM Self-Evaluation},
  author={Wang, Yiming and Zhang, Pei and Yang, Baosong and Wong, Derek F and Wang, Rui},
  journal={arXiv preprint arXiv:2410.13640},
  year={2024}
}

@article{wang2024embedding,
  title={Embedding trajectory for out-of-distribution detection in mathematical reasoning},
  author={Wang, Yiming and Zhang, Pei and Yang, Baosong and Wong, Derek and Zhang, Zhuosheng and Wang, Rui},
  journal={Advances in Neural Information Processing Systems},
  volume={37},
  pages={42965--42999},
  year={2024}
}

@article{brown2020language,
  title={Language models are few-shot learners},
  author={Brown, Tom and Mann, Benjamin and Ryder, Nick and Subbiah, Melanie and Kaplan, Jared D and Dhariwal, Prafulla and Neelakantan, Arvind and Shyam, Pranav and Sastry, Girish and Askell, Amanda and others},
  journal={Advances in neural information processing systems},
  volume={33},
  pages={1877--1901},
  year={2020}
}

@inproceedings{roy2007effective,
  title={The effective rank: A measure of effective dimensionality},
  author={Roy, Olivier and Vetterli, Martin},
  booktitle={2007 15th European signal processing conference},
  pages={606--610},
  year={2007},
  organization={IEEE}
}

@article{zang2024modeling,
  title={Modeling Selective Feature Attention for Representation-based Siamese Text Matching},
  author={Zang, Jianxiang and Liu, Hui},
  journal={arXiv preprint arXiv:2404.16776},
  year={2024}
}

@article{dou2024multi,
  title={Multi-Programming Language Sandbox for LLMs},
  author={Dou, Shihan and Zhang, Jiazheng and Zang, Jianxiang and Tao, Yunbo and Zhou, Weikang and Jia, Haoxiang and Liu, Shichun and Yang, Yuming and Xi, Zhiheng and Wu, Shenxi and others},
  journal={arXiv preprint arXiv:2410.23074},
  year={2024}
}

@inproceedings{zang2024explanation,
  title={Explanation based bias decoupling regularization for natural language inference},
  author={Zang, Jianxiang and Liu, Hui},
  booktitle={2024 International Joint Conference on Neural Networks (IJCNN)},
  pages={1--8},
  year={2024},
  organization={IEEE}
}

@inproceedings{zang2023extract,
  title={How to extract and interact? nested siamese text matching with interaction and extraction},
  author={Zang, Jianxiang and Liu, Hui},
  booktitle={International Conference on Artificial Neural Networks},
  pages={523--535},
  year={2023},
  organization={Springer}
}

@incollection{zang2023improving,
  title={Improving text semantic similarity modeling through a 3d siamese network},
  author={Zang, Jianxiang and Liu, Hui},
  booktitle={ECAI 2023},
  pages={2970--2977},
  year={2023},
  publisher={IOS Press}
}
\appendix

\section{Related Work and Further Analysis}

\subsection{Evaluation of Language Models}

\begin{figure*}[t]
\centering
\includegraphics[width=1\textwidth]{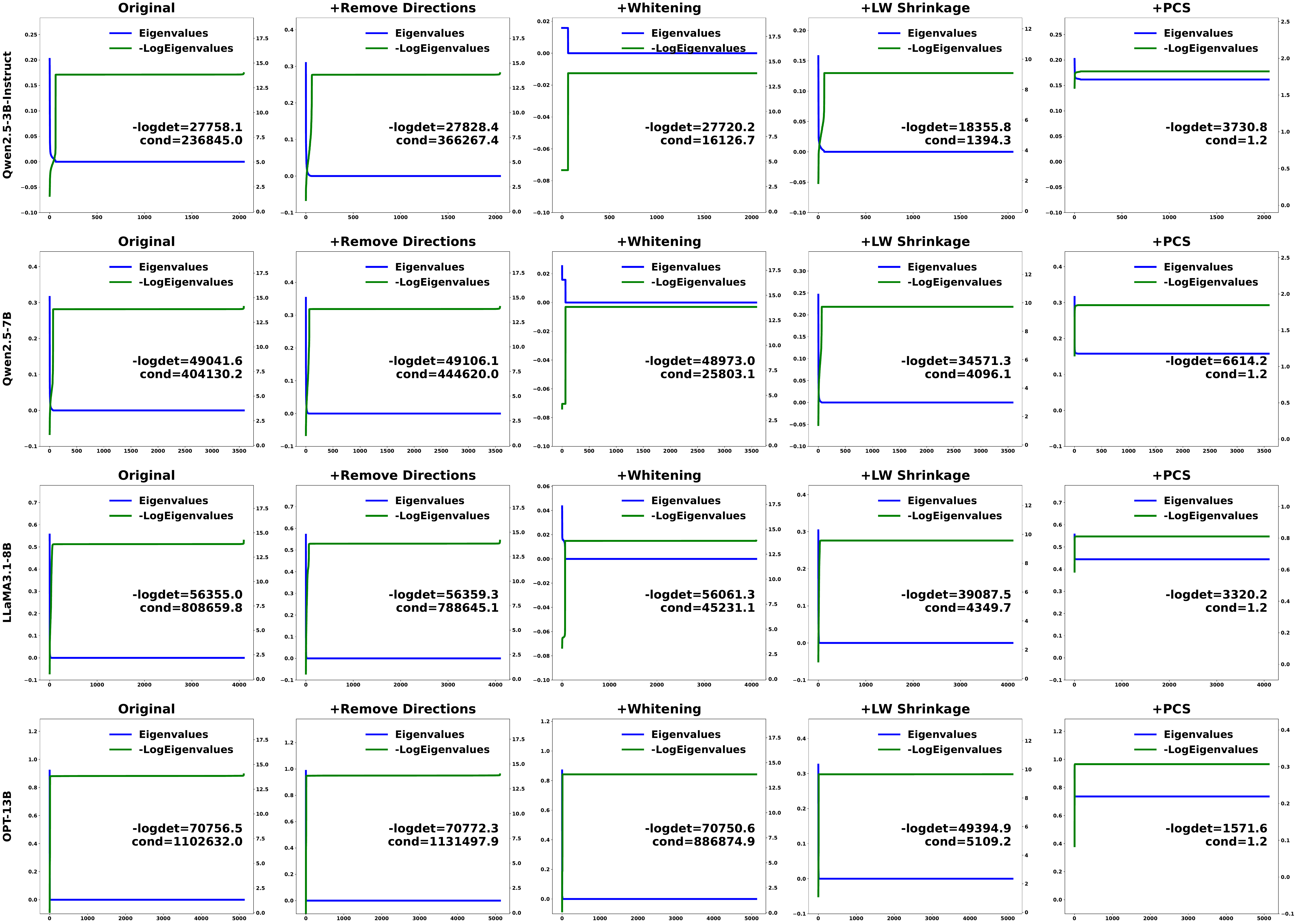}
\caption{The eigenvalues and their negative logarithmic distributions of different models' representations before and after processing with different anisotropy razors.}\label{fig.multi_model_eigen}
%不同模型的表征经过不同anisotropy razors后处理前后的特征值以及其负对数分布。
\end{figure*}

The evaluation of language models is currently in a state of rapid iterative development, encompassing a variety of tasks, datasets, and benchmarks~\cite{celikyilmaz2020evaluation,zheng2023judging,tan2024can}. Traditional evaluation metrics such as accuracy, F1-score~\cite{sasaki2007truth}, BLEU~\cite{sellam2020bleurt}, and ROUGE~\cite{lin2004rouge} focus on comparing model predictions with annotated labels in downstream tasks. Other metrics like perplexity and cross-entropy loss do not rely on annotated labels and are computed solely based on input text. However, these methods primarily emphasize external evaluation based on model predictions.  

Recently, significant efforts have been devoted to exploring the mechanisms by which language models (LMs) process information internally, driving the development of LM self-evaluation~\cite{wei2024diff,wang2024embedding,wang2024latent,zang2024explanation,dou2024multi} independent of specific tasks and model outputs.
The concept of ``compression as intelligence'' has provided an information-theoretic internal evaluation perspective for language models, highlighting that the acquisition of world knowledge by language models is a denoising process~\cite{sutskever2023compressors,deletanglanguage,wei2024diff,chen2025information}. Differential entropy of representations, as a classical information-theoretic measure, effectively quantifies the internal uncertainty of language models~\cite{cheninside,zhouyin2023understanding}. Semantic volume~\cite{li2025semantic} leverages representation-level differential entropy-aware compression metrics to offer a novel perspective for language model evaluation. However, related work has found that such compression can only model the scale of language models and fails to align with their capabilities~\cite{zang2023improving,wei2024diff,li2025semantic}.  

We introduce the concept of compression hacking in language model representations, where the noisy directions of LM representations sacrifice spatial uniformity to feign high compression rates. This implies that we can refine the information compression perspective by considering the geometric distortions in the language model's representation space.

% \subsection{Compression based Metrics for LLM Evaluation}
% % Information theory is a crucial tool for analyzing the inner workings of neural networks. The information bottleneck framework~\cite{tishby2000information,tishby2015deep} helps explain supervised learning and provides guidance for understanding and improving semi-supervised and self-supervised learning~\cite{skean2023dime,taninformation,zhangmatrix}. 

% For language models, prior work has utilized information theory to analyze hidden representations, often estimating the information contained in pretrained language models through probes trained on specific downstream tasks~\cite{voita2020information,hewitt2021conditional,pimentel2020information}.  
% Undoubtedly, information theory also provides a solid theoretical foundation for evaluating the reliability of LMs' self evaluation. Token-level uncertainty measures~\cite{malininuncertainty,kuhnsemantic,duan2023shifting}, which often involve predicting the confidence or entropy of output tokens, have been widely explored. Additionally, Consistency based methods play an essential role in reliability assessment~\cite{raj2023semantic,manakul2023selfcheckgpt,wei2024diff}. 

% However, all these entropy-based metrics assume the reliability of the representation space. Therefore, it is necessary to consider a metric that accounts for the anisotropy of the LM's representation space to ensure the reliability of these entropy-based measures.

\subsection{Anisotropy of Language Models}

\textbf{Anisotropy} The anisotropy of language models reflects the geometric properties of the contextual embedding space. Related studies have observed that during sampling, the spatial embeddings of negative samples exhibit anisotropy, which describes how vectors are distributed within the contextual space~\cite{mimno2017strange,ethayarajh2019contextual}. The researchers found that most vectors occupy a relatively narrow cone within the space, and that vectors within this cone tend to have high cosine similarity~\cite{gaorepresentation}. Demeter pointed out that using softmax introduces structural weaknesses in the representation space, leading to bias, a common issue in language models~\cite{demeter2020stolen}.  
To better quantify the anisotropy of LMs, related work has identified isolated clusters and low-dimensional manifolds in the contextual embedding space, introducing tools for their qualitative and quantitative analysis~\cite{ethayarajh2019contextual,cai2021isotropy,rudman2022isoscore}. 
However, these tools are mainly based on similarity calculations of embedded representations. What is needed instead is an anisotropy metric that can establish a connection with entropy-based compression metric.

\noindent\textbf{Anisotropy Razors} To mitigate anisotropy in language models, existing research has proposed various solutions. Contrastive learning has emerged as a powerful tool for obtaining effective sentence representations, effectively reducing anisotropy by increasing the spatial distance between positive and negative samples~\cite{gao2021simcse,zhang2022unsupervised,jiang2022promptbert,zang2023extract,zang2024modeling}. In this work, we employ post-processing methods applied directly to the representation space as baseline approaches for the anisotropy razor:

\begin{itemize}
    \item \textbf{Remove Directions}~\cite{mu2018all}: First, subtract the common mean vector of all word vectors to eliminate global bias; then remove the top high-variance principal component directions via Principal Component Analysis (PCA). This process enhances semantic feature discriminability by eliminating non-semantic common information from word vectors, making the word space distribution more isotropic.
    \item \textbf{Whitening}~\cite{su2021whitening}: Zero-center the representations and transform the covariance matrix into an identity matrix, forcing the embedding distribution toward isotropy.
    \item \textbf{LW Shrinkage}~\cite{ledoit2004well}: Linearly shrink the sample covariance matrix toward the diagonal matrices to reduce noise interference in high-dimensional data, yielding more stable covariance matrix estimates. This operation mitigates excessive sensitivity in specific directions, promoting isotropic feature distributions.
\end{itemize}

These training-free paradigms provide references for decoupling anisotropy from compression. However, these methods maintain the linear geometric structure of the data, with eigenvalues still exhibiting consistent partitioning behavior. Figure~\ref{fig.multi_model_eigen} demonstrates the distribution changes in eigenvalues and their negative logarithms after applying these baseline anisotropy razor post-processing methods. The results show that the distributions after Remove Directions, Whitening, and LW-Shrinkage treatments retain their original forms, leaving cross-model relationships of the modified compression metrics relatively unchanged. Consequently, we propose principal component smoothing to force eigenvalues toward dominant features. As shown in Figure~\ref{fig.multi_model_eigen}, this approach induces significant changes in eigenvalue distributions.

\section{Statistical Properties of Principal Component Smoothing}

\begin{lemma}[Asymptotic Optimality of Ledoit-Wolf Shrinkage~\cite{ledoit2004well}]  
\label{lemma.LW_optimality}
Let $\Sigma \in \mathbb{R}^{D\times D}$ be the population covariance matrix and $\Sigma_{\mathbf{Z}} = \frac{1}{|\mathcal{V}|}\mathbf{Z}^\top\mathbf{Z}$ the sample covariance. The Ledoit-Wolf estimator 
\begin{equation}
\hat{\Sigma}_{\text{LW}} = (1-\beta_{\text{LW}})\Sigma_{\mathbf{Z}} + \beta_{\text{LW}}\mu\mathbf{I}, \quad \mu = \frac{1}{D}\operatorname{tr}(\Sigma_{\mathbf{Z}})
\end{equation}
attains minimal MSE when the shrinkage intensity satisfies $\beta_{\text{LW}} \asymp \frac{1}{|\mathcal{V}|}$. Under general covariance structures (without spectral sparsity), this yields asymptotic MSE:
\begin{equation}
\text{MSE}(\hat{\Sigma}_{\text{LW}}) \asymp \mathcal{O}\left(\frac{D}{|\mathcal{V}|}\right)
\end{equation}
\end{lemma}

\begin{theorem}[Statistical Stability of the Principal Component Smoothing Estimator]
\label{theorem.pcs}
Assume the true covariance matrix $\Sigma$ has a dominant eigenvalue 
\(\lambda_1^* = \max_d \lambda_d \gg \lambda_d^* \quad (d \geq 2)\),
i.e., spectral sparsity holds. Define the improved shrinkage estimator as:
\begin{equation}
\begin{aligned}
&\hat{\Sigma}_{\text{PCS}} \\&= (1-\beta_{\text{PCS}})\Sigma_{\mathbf{Z}} + \beta_{\text{PCS}}\lambda_1 \mathbf{I}, 
\beta_{\text{PCS}} \asymp \mathcal{O}\left( \frac{1}{\sqrt{|\mathcal{V}|}} \right) 
\end{aligned}
\end{equation}
where $\lambda_1$ is the largest eigenvalue of the sample covariance matrix $\Sigma_{\mathbf{Z}}$ and satisfies 
$\lambda_1 \xrightarrow{|\mathcal{V}|} \lambda_1^*$ in probability. When the sample size $|\mathcal{V}|$ is sufficiently large,
\begin{equation}
\operatorname{MSE}(\hat{\Sigma}_{\text{PCS}}) < \operatorname{MSE}(\hat{\Sigma}_{\text{LW}})
\end{equation}
\end{theorem}

\begin{proof}
We commence by analyzing the mean squared error (MSE) structure of covariance matrix estimators. Let $\|\cdot\|_F$ denote the Frobenius norm, the MSE decomposes into bias and variance components:
\begin{equation}
\text{MSE}(\hat{\Sigma}) = \underbrace{\left\| \mathbb{E}[\hat{\Sigma}] - \Sigma \right\|_F^2}_{\text{Bias}^2} + \underbrace{\mathbb{E}\left[ \left\| \hat{\Sigma} - \mathbb{E}[\hat{\Sigma}] \right\|_F^2 \right]}_{\text{Variance}}.
\end{equation}

For the Ledoit-Wolf estimator $\hat{\Sigma}_{\text{LW}} = (1-\beta_{\text{LW}})\Sigma_{\mathbf{Z}} + \beta_{\text{LW}}\mu\mathbf{I}$, under spectral sparsity $\lambda_1^* \gg \sum_{d=2}^D\lambda_d^*/D$, the shrinkage target $\mu \approx \lambda_1^*/D$ creates dominant bias from the leading eigenvalue:

\begin{equation}
\begin{aligned}
\text{Bias}_{\text{LW}}^2 &\approx \beta_{\text{LW}}^2 \left\| \Sigma - \mu\mathbf{I} \right\|_F^2 \\&= \beta_{\text{LW}}^2\left[ (\lambda_1^* - \mu)^2 + \sum_{d=2}^D(\lambda_d^* - \mu)^2 \right] \\&\asymp \beta_{\text{LW}}^2(\lambda_1^*)^2\left(1-\tfrac{1}{D}\right)^2
\end{aligned}
\end{equation}

According to lemma~\ref{lemma.LW_optimality}, the variance term inherits from sample covariance matrix with dimension scaling:

\begin{equation}
\text{Variance}_{\text{LW}} \approx (1-\beta_{\text{LW}})^2\cdot\mathcal{O}\left(\tfrac{D^2}{|\mathcal{V}|}\right) \asymp \mathcal{O}\left(\tfrac{D^2}{|\mathcal{V}|}\right)
\end{equation}
where the $\mathcal{O}(D^2/|\mathcal{V}|)$ scaling comes from concentration of sample covariance in high dimensions.

For our eigenvalue-shrinkage estimator $\hat{\Sigma}_{\text{PCS}} = (1-\beta_{\text{PCS}})\Sigma_{\mathbf{Z}} + \beta_{\text{PCS}}\lambda_1\mathbf{I}$, the preserved leading eigenvalue estimation $\lambda_1 \xrightarrow{p} \lambda_1^*$ fundamentally alters the bias-variance tradeoff. The bias now originates from minor eigenvalues:
\begin{equation}
\text{Bias}_{\text{PCS}}^2 = \beta_{\text{PCS}}^2\sum_{d=2}^D(\lambda_d^* - \lambda_1^*)^2 \asymp \beta_{\text{PCS}}^2(D-1)(\lambda_1^*)^2
\end{equation}
where the last approximation uses $\lambda_d^* \ll \lambda_1^*$ from spectral sparsity. The variance term splits into two parts:
\begin{equation}
\begin{aligned}
\text{Variance}_{\text{PCS}} &= (1-\beta_{\text{PCS}})^2\underbrace{\text{Var}\left(\sum_{d=2}^D\lambda_d\right)}_{\asymp \mathcal{O}\left(\tfrac{(D-1)\lambda_1^{*2}}{|\mathcal{V}|}\right)} \\&+ \beta_{\text{PCS}}^2\underbrace{\text{Var}(\lambda_1)}_{\asymp \mathcal{O}\left(\tfrac{\lambda_1^{*2}}{|\mathcal{V}|}\right)}    
\end{aligned}
\end{equation}
With optimal shrinkage intensity $\beta_{\text{PCS}} = \mathcal{O}(1/\sqrt{|\mathcal{V}|})$, the dominant variance term becomes:
\begin{equation}
\text{Variance}_{\text{PCS}} \asymp \mathcal{O}\left(\tfrac{(D-1)\lambda_1^{*2}}{|\mathcal{V}|}\right).
\end{equation}

The MSE comparison reveals fundamental differences in scaling laws. For $\hat{\Sigma}_{\text{LW}}$ with $\beta_{\text{LW}} = \mathcal{O}(1/|\mathcal{V}|)$:
\begin{equation}
\text{MSE}(\hat{\Sigma}_{\text{LW}}) \asymp \underbrace{\mathcal{O}\left(\tfrac{\lambda_1^{*2}}{|\mathcal{V}|^2}\right)}_{\text{Bias}^2} + \underbrace{\mathcal{O}\left(\tfrac{D^2}{|\mathcal{V}|}\right)}_{\text{Variance}}.
\end{equation}
For $\hat{\Sigma}_{\text{PCS}}$ with dimension-adaptive shrinkage:
\begin{equation}
\text{MSE}(\hat{\Sigma}_{\text{PCS}}) \asymp \underbrace{\mathcal{O}\left(\tfrac{(D-1)\lambda_1^{*2}}{|\mathcal{V}|}\right)}_{\text{Bias}^2} + \underbrace{\mathcal{O}\left(\tfrac{(D-1)\lambda_1^{*2}}{|\mathcal{V}|}\right)}_{\text{Variance}}.
\end{equation}

When $|\mathcal{V}| \to \infty$, the $\mathcal{O}(1/|\mathcal{V}|)$ terms dominate $\mathcal{O}(1/|\mathcal{V}|^2)$. Under spectral sparsity $\lambda_1^* \gg \lambda_d^*$ ($d\geq2$), the improvement ratio becomes:
\begin{equation}
\frac{\text{MSE}(\hat{\Sigma}_{\text{PCS}})}{\text{MSE}(\hat{\Sigma}_{\text{LW}})} \asymp \frac{D\lambda_1^{*2}/|\mathcal{V}|}{D^2/|\mathcal{V}|} = \frac{\lambda_1^{*2}}{D} \ll 1,
\end{equation}
where the inequality follows from $\lambda_1^{*2}/D \leq (\sum_{d=1}^D\lambda_d^*)^2/D^2$ by Cauchy-Schwarz. 
\end{proof}

\section{Significance Analysis}

\begin{table}[t]
\centering
\begin{tabular}{l|cc}
\hline\hline
\textbf{Model}        & \textbf{R²} & \textbf{p-value} \\ \hline
LLaMA3.1-8B           & 0.89        & ****             \\ 
LLaMA3.1-8B-Instruct  & 0.79        & ***              \\ 
LLaMA3.2-1B           & 0.80        & ****             \\ 
LLaMA3.2-1B-Instruct  & 0.78        & ***              \\ 
LLaMA3.2-3B           & 0.89        & ****             \\ 
LLaMA3.2-3B-Instruct  & 0.77        & ****             \\
OPT-0.125B            & 0.88        & ****             \\ 
OPT-1.3B              & 0.76        & ****             \\ 
OPT-2.7B              & 0.66        & **                \\ 
OPT-6.7B              & 0.91        & ****             \\ 
OPT-13B               & 0.80        & ***              \\ 
OPT-30B               & 0.83        & ****             \\ 
Qwen2.5-0.5B-Instruct & 0.81        & ****             \\ 
Qwen2.5-1.5B-Instruct & 0.86        & ***              \\ 
Qwen2.5-3B-Instruct   & 0.80        & ****             \\ 
Qwen2.5-7B-Instruct   & 0.79        & ****             \\ 
Qwen2.5-14B-Instruct  & 0.85        & ****             \\ 
Qwen2.5-32B-Instruct  & 0.83        & ****             \\ \hline\hline
\end{tabular}
\caption{The R² and p-values of the compression-anisotropy regression fitting curves across different models, where, **, ***, and **** denote statistical significance at the 1\%, 0.1\%, and 0.01\% levels respectively.}\label{fig.s1}
%不同模型关于各向异性的压缩回归拟合曲线的R2和p值。其中**、***、***分别是再1%、0.1%、0.01的显著性水平下有统计意义。
\end{table}

\begin{figure}[t]
\centering
\includegraphics[width=1\linewidth]{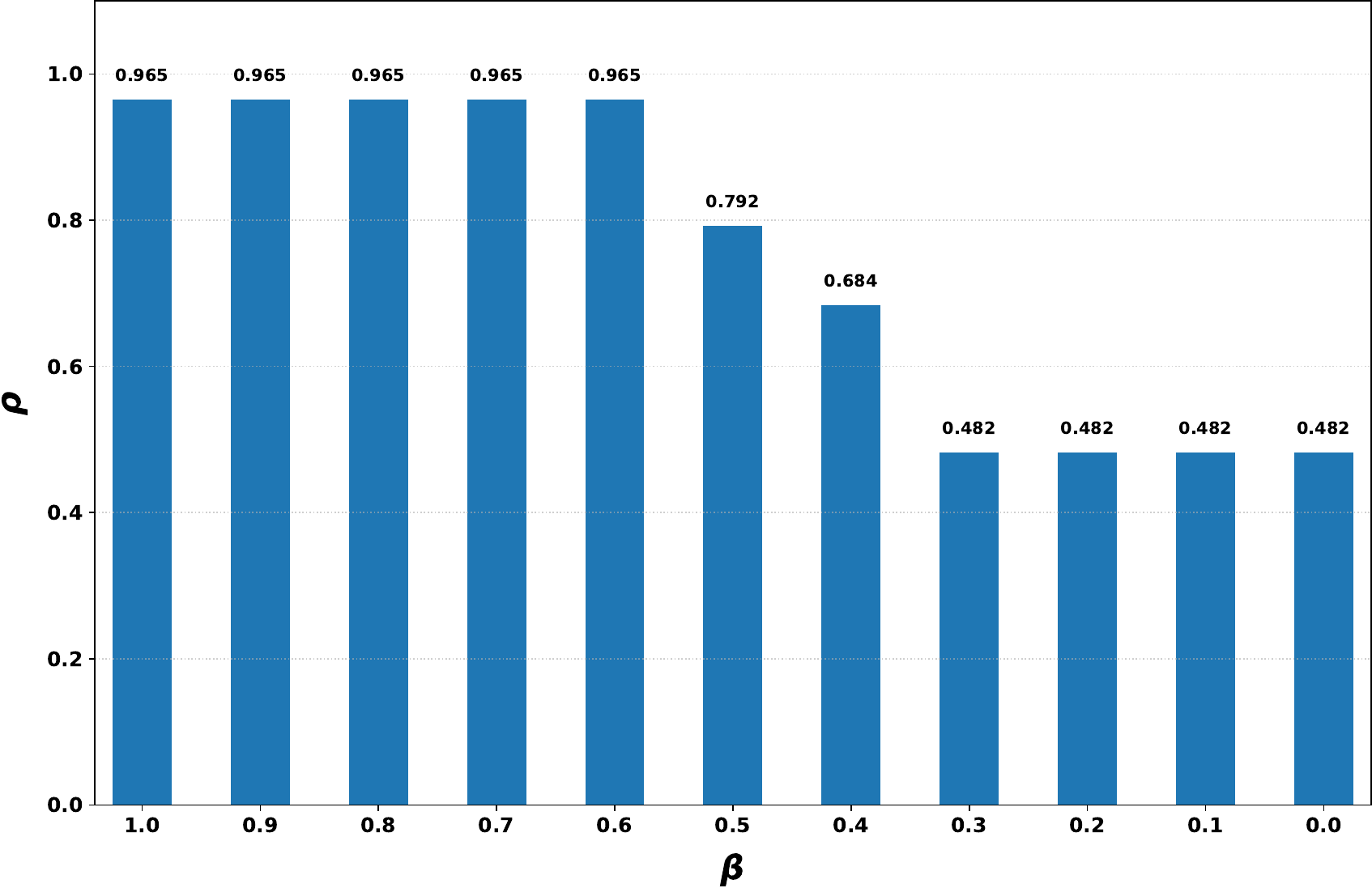}
\caption{The correlation coefficients between compression (PCS) and ground truth under different smoothing coefficients.}\label{fig.beta}
%不同模型关于各向异性的压缩回归拟合曲线，以及它们之间的Mann-Whitney U检验
\end{figure}

Our evaluation results presented in Table~\ref{fig.s1} demonstrate a strong and statistically significant relationship between compression and anisotropy across the 18 open-source language models examined. The high R² values (ranging from 0.7 to 0.9 for most models) indicate that linguistic anisotropy accounts for a substantial proportion of the observed compression phenomena. Furthermore, the compression-anisotropy synchronization proves statistically significant at stringent confidence levels (p<0.001 or p<0.01) for the majority of models. These robust and consistent findings across diverse architectures provide compelling empirical evidence that compression hacking is not merely an artifact but rather an intrinsic and fundamental characteristic of language model representations, revealing important insights about their underlying geometric properties.

\begin{figure}[h]
\centering
\includegraphics[width=1\linewidth]{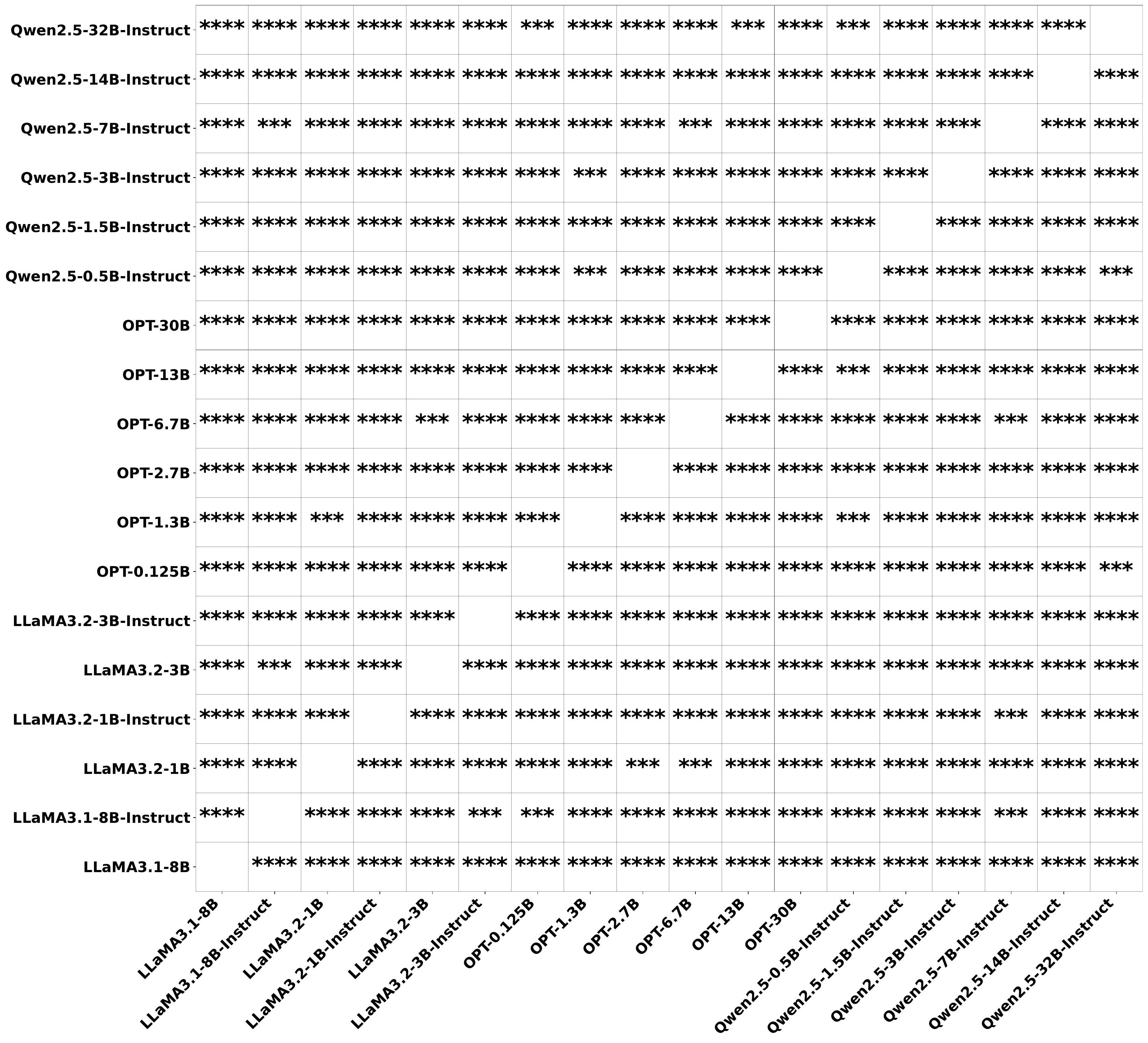}
\caption{The Mann-Whitney U tests of compression-anisotropy regression fitting between different models, where *** and **** denote statistical significance at the 0.1\% and 0.01\% levels respectively.}\label{fig.s2}
% 不同模型之间compression-anisotropy regression fitting的Mann-Whitney U检验。其中**、***、***分别是再1%、0.1%、0.01的显著性水平下有统计意义。
\end{figure}

Figure~\ref{fig.s2} presents the Mann-Whitney U test results for compression-anisotropy regression fitting across different models. Our analysis reveals that the differences between most model pairs achieve statistical significance at rigorous levels. These statistically significant variations in compression-anisotropy fitting curves demonstrate that the information compression metric, when adjusted for compression hacking effects, can effectively capture meaningful distinctions in model capabilities. This finding provides empirical validation that our refined compression-based evaluation framework offers discriminative power for comparing performance differences across language model architectures.

\begin{figure*}[h]
\centering
\includegraphics[width=1\textwidth]{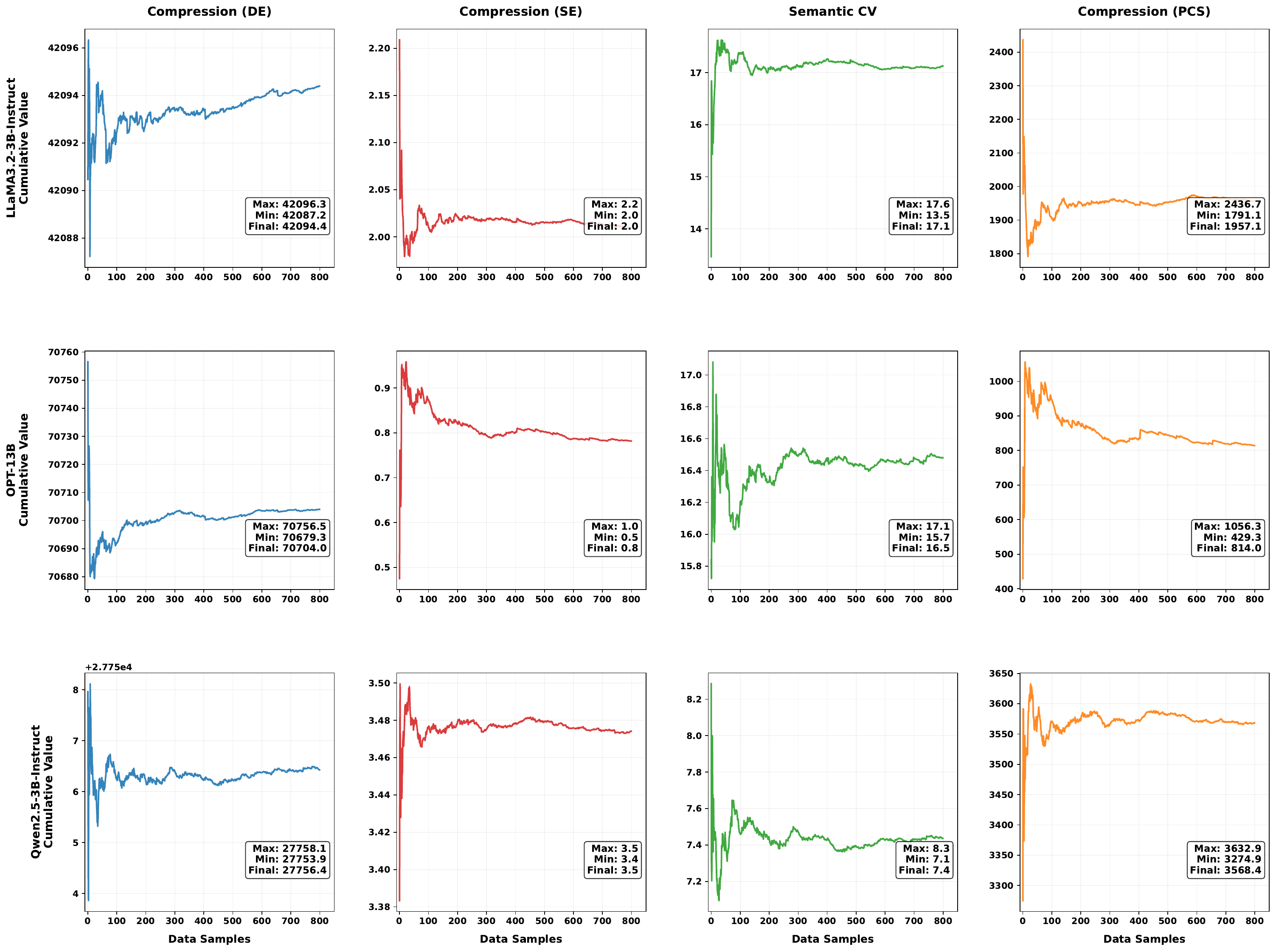}
\caption{The cumulative expected values of different metrics as the number of samples increases.}\label{fig.sample}
\end{figure*}

\section{Implementation Details of the Evaluation Pipeline}\label{sec.eval_detail}

For the projection dataset, we primarily collected 1,000 data samples from the pretraining corpus (Wiki~\cite{wikimedia_downloads}) and the instruction-tuning dataset (Dolly-15k~\cite{conover2023free}) to derive projection data. By sampling the in-context representations of these data points, we aim to estimate the full model’s representation space, ensuring the convergence of our metrics. Our pipeline defaults to sampling 800 data samples. Figure~\ref{fig.sample} illustrates the cumulative expected values of different metrics as the number of samples increases. We observe that all metrics converge relatively early to stable values, demonstrating that our refined metrics enable robust evaluation based on the provided projection dataset.

For the hyperparameter \(\alpha\) that ensures full-rank covariance matrices, we selected \(10^{-8}\). Regarding the smoothing coefficient (\(\beta\)) for principal component smoothing, we determined the interval \([0.6,1]\) to be appropriate. Figure~\ref{fig.sample} illustrates how different choices of principal component smoothing coefficients affect the compression (PCS). It can be observed that when \(\beta\) falls within \([0.6,1]\), the results maintain strong correlation with the ground truth. This occurs because the principal directions already dominate the compression computation. As the smoothing coefficient decreases, noise directions gradually regain prominence in the compression calculation.

% Fudan University: Tao Gui*, Qi Zhang, Xuanjing Huang
% Shanghai Key Lab of Intelligent Information Processing: Tao Gui, Qi Zhang, Xuanjing Huang
% Pengcheng Laboratory: Tao Gui

\end{document}